\documentclass[letterpaper]{article} 
\usepackage{aaai25}  
\usepackage{times}  
\usepackage{helvet}  
\usepackage{courier}  
\usepackage[hyphens]{url}  
\usepackage{graphicx} 
\usepackage{natbib}  
\usepackage{caption} 
\usepackage{multicol}
\usepackage{amsthm,amsmath,amssymb}
\usepackage{enumitem}
\usepackage{algorithm}
\usepackage{algpseudocode}
\usepackage{newfloat}
\usepackage{listings}
\usepackage{anyfontsize}

\usepackage[utf8]{inputenc}
\usepackage{booktabs}
\usepackage{makecell}
\usepackage{amsfonts}
\usepackage{nicefrac}
\usepackage{microtype}
\usepackage{appendix}
\usepackage{amsmath}
\usepackage{amssymb}
\usepackage{mathtools}
\usepackage{amsthm}
\usepackage{bm}
\usepackage{thmtools}
\usepackage{thm-restate}
\usepackage{multirow}
\usepackage[capitalize,noabbrev]{cleveref}
\usepackage{graphicx}
\usepackage{bbding}
\usepackage{bigints}
\usepackage{soul}
\usepackage{makecell}

\usepackage{caption, subcaption}
\usepackage[table]{xcolor}
\definecolor{baselinecolor}{gray}{.8}
\newcommand{\baseline}[1]{\cellcolor{baselinecolor}{#1}}
\newlength\savewidth\newcommand\shline{\noalign{\global\savewidth\arrayrulewidth
  \global\arrayrulewidth 1pt}\hline\noalign{\global\arrayrulewidth\savewidth}}
\newcommand{\tablestyle}[2]{\setlength{\tabcolsep}{#1}\renewcommand{\arraystretch}{#2}\centering\footnotesize}

\newcolumntype{x}[1]{>{\centering\arraybackslash}p{#1pt}}
\newcolumntype{y}[1]{>{\raggedright\arraybackslash}p{#1pt}}
\newcolumntype{z}[1]{>{\raggedleft\arraybackslash}p{#1pt}}

\newtheorem{definition}{Definition}
\DeclareCaptionStyle{ruled}{labelfont=normalfont,labelsep=colon,strut=off} 
\floatstyle{ruled}
\newtheorem{proposition}{Proposition}
\newtheorem{assumption}{Assumption}
\newtheorem{properties}{Properties}
\newcolumntype{"}{!{\vrule width 1pt}}
\usepackage{newfloat}
\usepackage{listings}
\DeclareCaptionStyle{ruled}{labelfont=normalfont,labelsep=colon,strut=off} 
\floatstyle{ruled}
\newfloat{listing}{tb}{lst}{}
\floatname{listing}{Listing}
\title{MEATRD: Multimodal Anomalous Tissue Region Detection\\Enhanced with Spatial Transcriptomics}

\author {
    Kaichen Xu\textsuperscript{\rm 1}\equalcontrib,
    Qilong Wu \textsuperscript{\rm 1}\equalcontrib,
    Yan Lu \textsuperscript{\rm 1},
    Yinan Zheng \textsuperscript{\rm 1},
    Wenlin Li \textsuperscript{\rm 1},
    Xingjie Tang \textsuperscript{\rm 1}, 
    Jun Wang \textsuperscript{\rm 2}, 
    Xiaobo Sun \textsuperscript{\rm 1}\thanks{Corresponding author.}
}
\affiliations {
    \textsuperscript{\rm 1} School of Statistics and Mathematics, Zhongnan University of Economics and Law\\
    \textsuperscript{\rm 2} iWudao Tech\\
    \{kaichenxu, qilongwu, yanlu, yinanzheng, wenlinli, xingjietang\}@stu.zuel.edu.cn, xsun28@gmail.com, jwang@iwudao.tech
}

\usepackage{bibentry}

\begin{document}

\maketitle

\begin{abstract}
The detection of anomalous tissue regions (ATRs) within affected tissues is crucial in clinical diagnosis and pathological studies. Conventional automated ATR detection methods, primarily based on histology images alone, falter in cases where ATRs and normal tissues have subtle visual differences. The recent spatial transcriptomics (ST) technology profiles gene expressions across tissue regions, offering a molecular perspective for detecting ATRs. However, there is a dearth of ATR detection methods that effectively harness complementary information from both histology images and ST. To address this gap, we propose MEATRD, a novel ATR detection method that integrates histology image and ST data. MEATRD is trained to reconstruct image patches and gene expression profiles of normal tissue spots (inliers) from their multimodal embeddings, followed by learning a one-class classification AD model based on latent multimodal reconstruction errors. This strategy harmonizes the strengths of reconstruction-based and one-class classification approaches. At the heart of MEATRD is an innovative masked graph dual-attention transformer (MGDAT) network, which not only facilitates cross-modality and cross-node information sharing but also addresses the model over-generalization issue commonly seen in reconstruction-based AD methods. Additionally, we demonstrate that modality-specific, task-relevant information is collated and condensed in multimodal bottleneck encoding generated in MGDAT, marking the first theoretical analysis of the informational properties of multimodal bottleneck encoding. Extensive evaluations across eight real ST datasets reveal MEATRD's superior performance in ATR detection, surpassing various state-of-the-art AD methods. Remarkably, MEATRD also proves adept at discerning ATRs that only show slight visual deviations from normal tissues. Our code is available at \url{https://github.com/wqlzuel/MEATRD}.
\end{abstract}

\section{Introduction}
Detecting anomalous tissue regions (ATR) within tissues from affected individuals is essential in clinical diagnostics, pathological studies, and targeted therapies ~\cite{srinidhi2021deep}. Traditionally, automated ATR detection, which typically applies computer vision techniques to histology images, e.g., whole-slide images (WSI) stained with hematoxylin and eosin (H\&E) \cite{zingman2023learning}, is a specialized task of image anomaly detection (AD). However, histology images, unlike natural images (e.g., those in ImageNet dataset)  \cite{Bergmann_2019_CVPR}, present unique challenges for AD due to their inherent high complexity \cite{zehnder2022multiscale}, subtle differences between ATRs and normal tissues \cite{shenkar2021anomaly}, the diverse manifestations of ATRs \cite{komura2018machine},  and variability in staining quality \cite{zingman2023learning}. The complexities demand supplementary information to visual cues for accurate ATR detection. 

Spatial transcriptomics (ST) meets this need by providing spatial gene expression data. By far, a total of 1033 publicly available human ST datasets that span 56 diseases and 35 tissues, providing a rich resource for investigating ATRs at the molecular level \cite{wang2024crost}. A typical ST dataset is structured as a matrix $\mathbf{X}\in \mathbb{R}^{N\times G}$, where $\mathbf{X}_{i,j}$ denotes the expression read counts of the $j$-th gene mapped to $i$-th tissue spot. As illustrated in \Cref{fig1}, these spots, ranging in size from 10 to 200 $\mu$m as per sequencing technology, are spatially arranged in arrays to cover the entire tissue slice ~\cite{hu2023deciphering}, thereby characterizing gene expression profile across the tissue. This molecular-level data, especially in cases where ATRs are visually similar to normal tissues, can significantly aid in their detection ~\cite{hu2021spagcn}. However, due to limitations inherent to sequencing technology, ST data suffer from severe noise and substantial missing values  in gene expression measurements \cite{wang2022sprod}, leading to compromised precision in demarcating tissue regions \cite{wang2022region}. Integration of histology images with ST data presents a promising solution to these challenges. As illustrated in our toy example in \Cref{fig1}, the blank spots in the ST dataset's spatial map, which represent tumor core locations with missing gene expression data, are visually identifiable in the accompanying histology image. Conversely, the tumor edge region, which may not be easily distinguishable from normal tissues visually, is detectable in the ST data. Therefore, the information from the two modalities can complement each other, greatly enhancing the precision of ATR detection. Fortunately, ST technologies like 10x Visium \cite{Moses2022st} provide accompanying histology images, allowing concurrent analysis of visual and genetic information for ATR detection.

\begin{figure}[t]
  \centering
  \includegraphics[width=\linewidth]{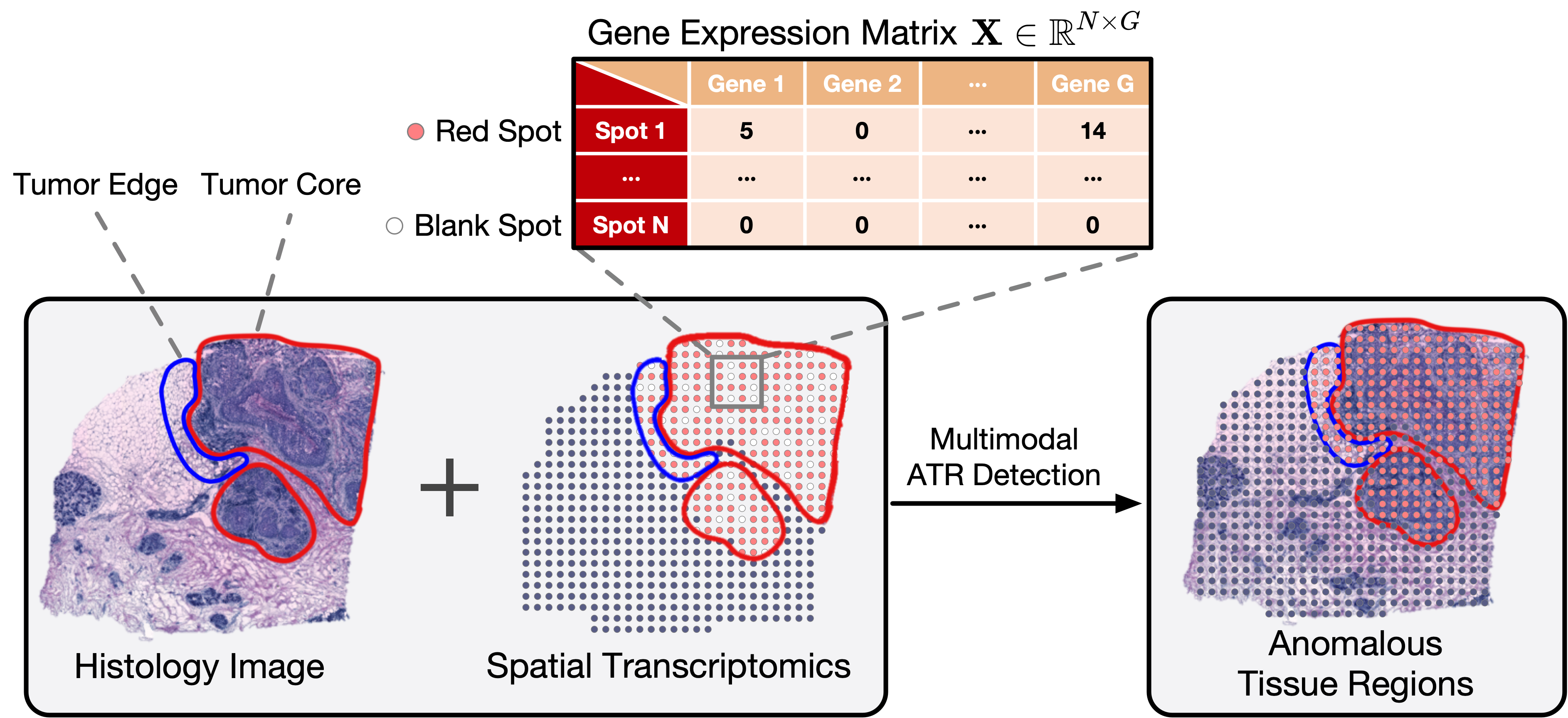}
  \caption{Detecting ATRs with histology images and ST data. ATRs include both tumor core and edge regions, as delineated by red and blue outlines in the histology image, respectively. The tumor edge region visually resembles the adjacent normal tissues. In the spatial map of the ST dataset, the ATRs encompass both red and blank spots, with blank spots indicating locations of missing gene expression data.}
  \label{fig1}
\end{figure}

Given the rarity and unpredictable heterogeneity of anomalies, AD in images is often framed as an unsupervised learning problem, where anomalies are not known a priori. Models are trained exclusively on reference datasets comprising inliers to understand "normality" at training time and identify deviations from this norm as anomalies at inference time \cite{liu2023simplenet,Bergmann_2019_CVPR}. Contemporary image AD methods, which use deep learning to learn initial representations of normal images \cite{shvetsova2021anomaly,liu2023simplenet}, often involve an encoder pre-trained on large natural image datasets \cite{deng2022anomaly,roth2022towards}. These representations are then used to either model the inlier distribution in latent space, as seen in one-class classification methods \cite{DeepSVDD18}, or to reconstruct inliers in reconstruction-based methods \cite{schlegl2019f}. Instances in the target dataset, which exhibit low probability in the inlier distribution or larger-than-expected reconstruction errors are deemed anomalous. 

Despite successes of these methods in areas such as manufacturing defect inspection, financial fraud detection, etc \cite{sohn2020learning}, the unique challenges posed by ATR detection require more specialized methods \cite{riasatian2021fine,tschuchnig2022anomaly}. To meet this need, adaptions made to image AD methods focus on representation learning and anomaly discrimination techniques. For example, image encoders pre-trained on natural images are replaced with those tailored for histology images, such as U-Net \cite{zehnder2022multiscale}, DenseNet \cite{riasatian2021fine}, and s2-AnoGAN \cite{pocevivciute2021unsupervised}. In addition, anomaly scoring is adapted to use perceptual loss instead of pixel-wise reconstruction errors commonly used for natural images \cite{shvetsova2021anomaly,zehnder2022multiscale}. However, these methods may struggle when ATRs visually resemble normal tissues \cite{bejnordi2017diagnostic}. In contrast, ST differentiates tissue regions at the gene expression level \cite{hu2021spagcn,dong2022deciphering}, providing a remedy for ATR detection involving such complexities. Currently, Spatial-ID \cite{shen2022spatial} is the only method that uses ST for ATR detection, employing a DNN classifier which assigns spots in the ST dataset to known regions while determining those with uncertain assignments as anomalies. However, this classification-based approach can induce high false positive rates, as uncertainties in assignment could stem from similarities among normal tissues rather than the presence of ATRs \cite{CAMLU}. Its sole reliance on ST data also makes it vulnerable to noise and dropouts in gene expression measurements, even for detecting visually recognizable ATRs. 

In this study, we propose \textbf{M}ultimodality \textbf{E}nhanced \textbf{A}nomalous \textbf{T}issue \textbf{R}egion \textbf{D}etection (\textbf{MEATRD}), the first method that integrates histology images and ST data for enhanced ATR detection. MEATRD conceptualizes tissue spots as nodes within an attributed graph, leveraging a reconstruction-based graph model for inlier nodes reconstruction from dual perspectives of image and gene expression. During inference, the discrepancies between reconstruction errors of inliers (i.e., normal tissues) and anomalies (i.e., ATRs) can be exploited by a discriminative model for accurate ATR detection. As shown in Figure 2, MEATRD involves three stages. \textbf{Stage I} focuses on extracting visual features of histology images. The histology image is segmented into a patch centered around each spot, which are processed into imagery embeddings. \textbf{Stage II} aims to reconstruct the gene expression profiles and image patches of each spot from their fused embeddings, obtained using our innovative masked graph dual-attention transformer (MGDAT) network.  MGDAT allows concurrent cross-node and cross-modal attention calculations, promoting efficient cross-modality information sharing and incorporation of spatial relationships among spots. Additionally, to counter potential model over-generalization\footnote{A common pitfall of reconstruction-based methods where anomalies might yield low reconstruction errors \cite{liu2023simplenet,ristea2022self}.}, we employ the node-feature masking strategy, which forces the model to rely more on the surrounding context and cross-modal information. \textbf{Stage III} focuses on acquiring a one-class classification model to identify anomalies. Unlike existing one-class classification AD methods that use instance deep embeddings and are prone to reference-target domain shifts \cite{ouardini2019towards}, our model pioneers in using domain shift-robust latent multimodal reconstruction losses \cite{donahue2016adversarial,schlegl2019f} for more reliable anomaly detection. By collapsing inliers' reconstruction losses into a compact hypersphere, our model increases the reconstruction error discrepancy between inliers and anomalies, thereby further mitigating model over-generalization. In summary, our main contributions include:
\begin{itemize}[left=0pt] \setlength{\itemsep}{0.01pt}
    \item We propose MEATRD, a pioneer multimodal method that integrates spatial transcriptomics with histology images for enhanced ATR detection.  
    \item MEATRD simultaneously addresses the over-generalization in reconstruction-based AD methods and the domain shift issue in one-class classification, leading to significant performance improvement.
    \item We design an MGDAT network as the core component of MEATRD to facilitate cross-modality and cross-node information exchange while ameliorating model over-generalization. We also demonstrate the theoretical foundation for this information exchange, which is grounded in MGDAT's ability to generate inclusive and condensed encoding of modality-specific, task-relevant information (supplementary material D). 
    \item Extensive benchmarks on eight breast cancer ST datasets demonstrate MEATRD's superiority over nine state-of-the-art (SOTA) AD methods in accurately detecting ATRs, including those with subtle visual deviations from surrounding normal tissues. 
\end{itemize}

\section{Preliminary}
\begin{figure*}[t]
  \centering
  \includegraphics[width=0.8\linewidth]{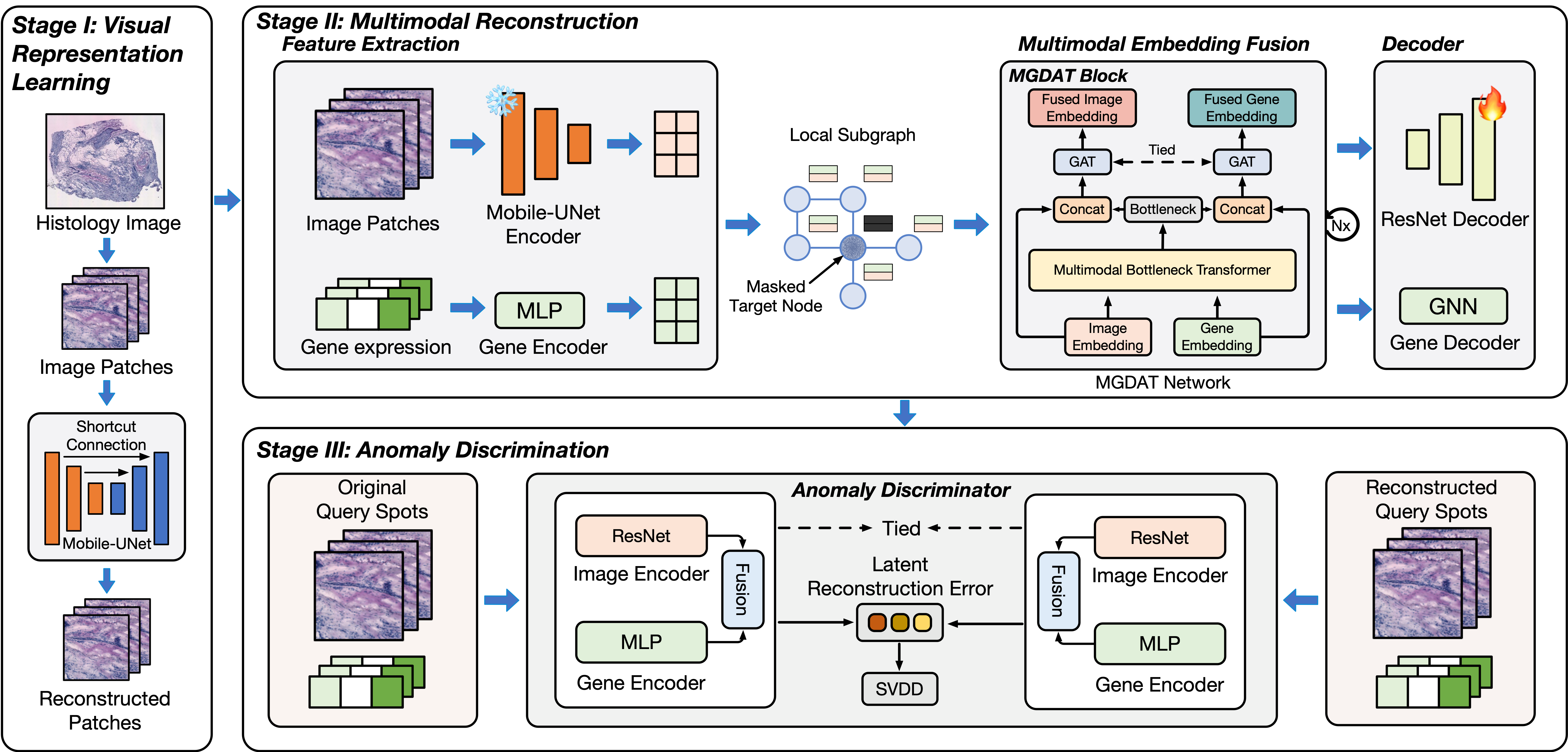}
  \caption{The workflow of MEATRD.}
  \label{fig2}
\end{figure*}

\begin{definition}
\textbf{ST Dataset and Associated Histology Image.} Let $\bm{X} \in \mathbb{R}^{N\times G}$ be a ST dataset, where $N$ is the number of tissue spots and $G$ is the number of genes. $S_N$ and $S_G$ denote the set of spots and genes, respectively. $\bm{X}_{i,j}$ represents the the read counts of gene $j$ at spot $i$, and $\bm{x}_i\in \mathbb{R}^G$ represents the gene expression profile at spot $i$. Let $\bm{P}\in \mathbb{R}^{H\times W\times C}$ be the associated histology image, where $H, W$, and $C$ are the height, width, and number of channels, respectively. 
\end{definition}

\begin{definition} \label{graph definition}
\textbf{Graph Representation of ST Dataset and Histology Image.} For a given ST dataset $\bm{X}$, the associated histology image $\bm{P}$ is segmented into $N$ patches, where $\bm{P}_i\in \mathbb{R}^{h\times w\times C}$ denotes the patch centered around spot $i\in S_N$, with height $h$ and width $w$. Then spots are modeled as nodes on an unweighted, attributed graph $G(S_N,A,\mathcal{\bm{Z}})$, where $A\in\{0,1\}^{N\times N}$ is the adjacency matrix, and $\mathcal{\bm{Z}}\coloneqq[\mathcal{\bm{Z}}_{image}||\mathcal{\bm{Z}}_{gene}]$ is the node feature matrix. $\mathcal{\bm{Z}}_{img}\in \mathbb{R}^{N\times D}$ and $\mathcal{\bm{Z}}_{gene}\in \mathbb{R}^{N\times D}$ are embeddings of image patches and gene expression profiles of spots. $A(i,j)=1$ if node $j\in n(i)$, where $n(i)$ is the set of $k$-nearest neighbors of node $i$, and $A(i,j)=0$ otherwise. $k$ is typically set to be six due to the hexagonal arrangement of spots ~\cite{xu2024detecting}.
\end{definition}

\begin{definition}
\textbf{Problem Definition.} Let $\mathcal{X}$ and $\mathcal{P}$ denote the target ST dataset and associated histology image, respectively. Similarly, let $\bm{X}$ and $\bm{P}$ denote the reference ST dataset and associated histology image, respectively. We define $y_i \in \{0,1\}$ as the label for spot $i$, where $y_i=1$ indicates an anomalous spot, and $y_i=0$ otherwise. Note, $y_i=0, \forall i \in \bm{X}$; $y_{j}\in \{0,1\}, \forall j \in \mathcal{X}$. The task of ATR detection is defined as identifying the subset of anomalous spots within the target dataset: $\mathbb{S}=\{\mathcal{x}_i|y_{\mathcal{x}_i}=1, \forall \mathcal{x}_i \in \mathcal{X}\}$, using a model trained exclusively on $\bm{X}$ and $\bm{P}$. \footnote{Related work is in supplementary material A due to space limitation.}
\end{definition}

\section{Method}

As shown in Figure 2, the workflow of MEATRD includes three stages: Stage I extract visual features from histology image patch of each spot; Stage II focuses on the learning of reconstructions of image patches and gene expression profiles using multimodal embeddings generated by a MGADT network; Stage III entails the training of an anomaly discriminator based on latent multimodal reconstruction errors.

\subsection{Extracting Visual Features of Histology Image Patches (Stage I)}
Initially, whole slide images are segmented into 32x32 patches centered around each spot in the ST dataset \cite{zong2022const}. The visual manifolds of these image patches are obtained using a Mobile-Unet, with an encoder consisting of downsampling convolutional layers and inverted residual blocks. Its decoder comprises upsampling deconvolutional layers and inverted residual blocks, connected to the encoder via shortcut connections. 

This design not only inherits the merits of U-Net in extracting visual features from histology images but also boosts computational efficiency by reducing the model's parameters. Given a histology image patch $\bm{P}_i$ for spot $i \in S_N$, the Mobile-Unet is pretrained to reconstruct it as $\widehat{\bm{P}}_i$, with a pretraining loss that is a mix of a perceptual loss, based on the Structural Similarity Index (SSIM), and an $L1$ reconstruction loss:
\begin{equation}
    \widehat{\bm{P}}_i \coloneqq D_1(E_1(\bm{P}_i)),\quad \bm{z}_i \in \mathbb{R}^D \coloneqq E_1(\bm{P}_i)
\end{equation} 
\begin{equation}
\mathcal{L}_{perc}=-\mathrm{SSIM}(\bm{P}_i, \widehat{\bm{P}}_i), \mathcal{L}_1=||\bm{P}_i-\widehat{\bm{P}}_i||_1
\end{equation}
\begin{equation}
\mathrm{SSIM}(\bm{X},\bm{Y})=\frac{(2\mu_{\bm{X}}\mu_{\bm{Y}}+C_1)(2\sigma_{\bm{X},\bm{Y}}+C_2)}{(\mu_{\bm{X}}^2+\mu_{\bm{Y}}^2+C_1)(\sigma_{\bm{X}}^2+\sigma_{\bm{Y}}^2+C_2)}
\end{equation} 
\begin{equation}\label{pretrain loss}
    \mathcal{L}_{pretrain}= \mathcal{L}_{perc}+ \mathcal{L}_1.
\end{equation}
where $\mu_{*}$ and $\sigma_{*}^2$ are the average intensity and variance of $* \in \{\bm{X},\bm{Y}\}$, respectively. $C_1$ and $C_2$ represent two constants to stabilize the division with a weak denominator. SSIM and $\mathcal{L}_1$ measure the structural similarities and pixel-by-pixel discrepancies between the original and reconstructed images, respectively. Then, pretraining loss enhances the representation learning of complex histology images by accounting for both contextual integrity, via $\mathcal{L}_{perc}$, and local details via $\mathcal{L}_1$ \cite{okada2021dreaming}. Following training, $E_1$ is used to yield image patch embeddings for each spot $i \in S_N$. Finally, unlike complex tissue images, which need to be converted into semantically meaningful representations in the first place, gene data have much clearer semantics. Therefore, MEATRD do not require a pretext representation learning stage for gene data. Rather, we use a two-layer MLP in stage II to rasterize gene data before feeding them into MGDAT blocks, where graph-based gene encoding takes places.

\subsection{Masked Graph Dual-Attention Transformer Network (Stage II)}
To generate information-rich multimodal spot embeddings for reconstruction, we fuse histology image patches and gene expression profiles while incorporating contextual inter-dependencies among spots to reveal their biological characteristics. This is achieved by modeling spots as nodes in an attributed graph $G(V,A,\mathcal{\bm{Z}})$, as described in \Cref{graph definition}, on top of which node representations are learned using an innovative masked graph attention network, termed MGDAT. This network, comprising a series of MGDAT blocks, allows information sharing across both data modality and graph nodes. Within each MGDAT block, nodes to be reconstructed are masked before aggregating fused gene and imagery attributes of their neighboring nodes via attention-based mechanism.  

Formally speaking, let $G_i(V_i,A_i,\mathcal{\bm{Z}}_i)$ denote the subgraph of a target node $i$ that covers up to its 3-hop neighbors, where $V_i,A_i$ and $\mathcal{\bm{Z}}_i$ denote the node set, adjacency matrix, and node attribute matrix of $G_i$, respectively. We set the number of hops to be 3 as using more hops will result in over-smoothing while fewer hops will significantly limit the information spread in the graph. $\bm{z}_{i}\in \mathbb{R}^D$ represents node $i$'s imagery attribute derived from Stage I, and $\bm{\zeta}_{i}\in \mathbb{R}^D$ represents node $i$'s gene attribute rasterised from its gene expression vector $\bm{x}_i$ using a two-layer MLP. $\bm{z}_{i}$ and $\bm{\zeta}_{i}$ are substituted with learnable mask tokens $\bm{z}_{[M]}\in \mathbb{R}^D$and $\bm{\zeta}_{[M]} \in \mathbb{R}^D$. 

This target-node-masking serves to prevent self-information leakage of the target node into its own embedding for reconstruction, thus alleviating the potential model over-generalization issue. $G_i$ is processed by the MGDAT network through its series of MGDAT blocks. For the $l$-th block, $l\in\{0,1,2\}$, the inputs are embeddings of the image patches, $\mathcal{\bm{Z}}_{img,i}^{(l)}\in\mathbb{R}^{|V_i|\times D}$, and the gene expression profiles, $\mathcal{\bm{Z}}_{gene,i}^{(l)}\in \mathbb{R}^{|V_i| \times D}$, of $V_i$. The initial embeddings are defined as $\mathcal{\bm{Z}}_{img,i}^{(0)}\coloneqq  [\bm{z}_1,..,\bm{z}_{[M]},..\bm{z}_{V_i}]^\top$ and $\mathcal{\bm{Z}}_{gene,i}^{(0)}\coloneqq [\bm{\zeta}_1,..,\bm{\zeta}_{[M]},..,\bm{\zeta}_{V_i}]^\top$. The $l$-th MGDAT block yields fused bottleneck embeddings $\mathcal{\bm{Z}}_{fb,i}^{(l)}\in \mathbb{R}^{|V_i|\times D'}, D'\ll D$ as follows: 
\begin{equation}\label{fb-transformer}
\mathcal{\bm{Z}}_{fb,i}^{(l)}=\mathrm{Trm}\left([\mathcal{\bm{Z}}_{img,i}^{(l)}||\mathcal{\bm{Z}}_{gene,i}^{(l)}];W_{Q}^{(l)},W_{K}^{(l)},W_{V}^{(l)}\right)
\end{equation}
where $\mathrm{Trm}$ denotes Transformer. $W_{Q}^{(l)},W_{K}^{(l)},W_{V}^{(l)}\in \mathbb{R}^{2D\times D'}$ are query, key and value parameters, respectively. $\mathcal{\bm{Z}}_{fb}^{(l)}$ serves as a bottleneck to collate and condense modality-specific, task-relevant information from image and ST data ~\cite{nagrani2021attention}, as theoretically demonstrated in supplementary material D. By concatenating $\mathcal{\bm{Z}}_{fb}^{(l)}$ with $\mathcal{\bm{Z}}_{img}^{(l)}$ and $\mathcal{\bm{Z}}_{gene}^{(l)}$, the two data modalities are bridged, facilitating access to their complementary task-relevant information. Next, multimodal information of $l$-hop neighbors is aggregated as follows:
\begin{equation}
h_{*,i}^{(l)} = [\mathcal{\bm{Z}}_{*,i}^{(l)}||\mathcal{\bm{Z}}_{fb,i}^{(l)}],\quad \text{where}\ * \in \{img, gene\},
\end{equation}
\begin{equation}
\alpha_{*,ij}^{(l)} = \frac{\exp(w_{att}^{(l)}\sigma(W^{(l)}[h_{*,i}^{(l)}||h_{*,j}^{(l)}]))}{\sum_{k\in \mathcal{N}_i} \exp(w_{att}^{(l)}\sigma(W^{(l)}[h_{*,i}^{(l)}||h_{*,k}^{(l)}])))},
\end{equation}

\begin{equation}\label{graph-attention}
\mathcal{\bm{Z}}_{*,i}^{(l+1)}= \sigma(\sum_{j\in \mathcal{N}_i} \alpha_{*,ij}^{(l)} W^{(l)} h_{*,j}^{(l)}),
\end{equation} 
where $\sigma$ denotes LeakyReLU, $w_{att}^{(l)}\in \mathbb{R}^D$ and $W^{(l)}\in \mathbb{R}^{D\times (D+D')}$ denote the attention weight matrix and regular weight matrix of the $l$-th MGDAT block, respectively. 

The histology image patches of $V_i$ are reconstructed from the final image embeddings, $\mathcal{\bm{Z}}_{img,i}$, through a ResNet-based deconvolutional decoder $D_2$, while the gene expression profiles of $V_i$ are reconstructed from the final gene embeddings, $\mathcal{\bm{Z}}_{gene,i}$, through a single-layer GNN-based decoder $D_3$ \cite{hou2023graphmae2}:
\begin{equation}
    \widetilde{\bm{P}}_i\coloneqq  D_2(\mathcal{\bm{Z}}_{img,i}), \quad \widetilde{\bm{x}}_i\coloneqq D_3(\mathcal{\bm{Z}}_{gene,i})
\end{equation}
The training loss of stage II includes an image-level loss, same as that defined in \Cref{pretrain loss}, and a gene-level reconstruction loss measured in scaled cosine errors:

\begin{equation}\label{rec error} \begin{split}
    \mathcal{L}_{rec}=&\alpha\cdot \sum_{i}^N(-\mathrm{SSIM}(\bm{P}_i, \widetilde{\bm{P}}_i)+||\bm{P}_i-\widetilde{\bm{P}}_i||_1) \\ &+ (1-\alpha)\cdot \sum_{i}^N \mathcal{L}_{SCE}(\bm{x}_i,\widetilde{\bm{x}}_i),
\end{split} \end{equation}
\begin{equation}
    \mathcal{L}_{SCE}(\bm{x}_i,\widetilde{\bm{x}}_i)=\left(1-\frac{\bm{x}_i^\top\widetilde{\bm{x}}_i}{||\bm{x}_i||\cdot||\widetilde{\bm{x}}_i||}\right)^\gamma, \gamma\ge 1,
\end{equation}
where $0<\alpha<1$ is the weight assigned to image reconstruction loss, $\gamma$ is a scaling factor. The workflow of Stage II is illustrated in \Cref{fig2} and Algorithm 1 in supplementary material C.

\begin{table*}[t]
\centering
\renewcommand{\arraystretch}{1.2}
\resizebox{\textwidth}{!}{
\begin{tabular}{c"c"cc"ccc"ccccc}
\Xhline{1.2pt}
\multirow{3}{*}{\centering \makecell{Target \\ Dataset}}& \multirow{3}{*}{\centering Metric}& \multicolumn{10}{c}{\large Method}\\

\Xcline{3-12}{1pt}
& & \multicolumn{2}{c"}{Multimodal-based} & \multicolumn{3}{c"}{Image-based} & \multicolumn{5}{c}{ST-based} \\ 

\Xcline{3-12}{1pt}
& & MEATRD & M3DM  & SimpleNet & f-AnoGAN & Patch SVDD & DOMINANT & PREM & Spatial-ID & scmap & CAMLU \\ 

\Xhline{1pt}
\multirow{2}{*}{10x-hBC-A1} & AUC & $\mathbf{0.756}_{\pm 0.007}$ & $0.520_{\pm 0.046}$ & $0.543_{\pm 0.095}$ & $\underline{0.642}_{\pm 0.109}$ & $0.614_{\pm 0.005}$ & $0.488_{\pm 0.117}$ & $0.211_{\pm 0.004}$ & $0.463_{\pm 0.067}$ & $0.500_{\pm 0.000}$ & $0.516_{\pm 0.021}$ \\

& F1 & $0.892_{\pm0.007}$ & $0.875_{\pm0.0013}$ & $0.875_{\pm0.011}$ & $0.892_{\pm0.017}$ & $\underline{0.892}_{\pm0.005}$ & $0.885_{\pm0.017}$ & $0.865_{\pm0.000}$ & $0.870_{\pm0.004}$ & $\mathbf{0.934}_{\pm0.000}$ & $0.376_{\pm0.383}$ \\
 
\Xhline{1pt}
\multirow{2}{*}{10x-hBC-B1} & AUC &  $\mathbf{0.920}_{\pm0.028}$ & $0.505_{\pm0.029}$ & $0.554_{\pm0.135}$ & $\underline{0.736}_{\pm0.144}$ & $0.442_{\pm0.025}$ & $0.698_{\pm0.077}$ & $0.288_{\pm0.006}$ & $0.195_{\pm0.083}$ & $0.504_{\pm0.000}$ & $0.667_{\pm0.160}$ \\

& F1 &  $\mathbf{0.841}_{\pm0.022}$ & $0.210_{\pm0.027}$ & $0.302_{\pm0.127}$ & $\underline{0.568}_{\pm0.176}$ & $0.225_{\pm0.025}$ & $0.352_{\pm0.143}$ & $0.073_{\pm0.008}$ & $0.067_{\pm0.064}$ & $0.354_{\pm0.000}$ & $0.428_{\pm0.365}$ \\ 

\Xhline{1pt}
\multirow{2}{*}{10x-hBC-C1} & AUC & $\mathbf{0.715}_{\pm0.017}$ & $0.540_{\pm0.034}$ & $0.501_{\pm0.099}$ & $0.485_{\pm0.035}$ & $0.401_{\pm0.0032}$ & $0.633_{\pm0.101}$ & $0.419_{\pm0.004}$ & $0.384_{\pm0.055}$ & $0.500_{\pm0.000}$ & $\underline{0.660}_{\pm0.156}$ \\ 

& F1 & $\mathbf{0.842}_{\pm0.021}$ & $0.735_{\pm0.028}$ & $0.735_{\pm0.024}$ & $0.713_{\pm0.021}$ & $0.661_{\pm0.005}$ & $0.769_{\pm0.040}$ & $0.695_{\pm0.006}$ & $0.687_{\pm0.013}$ & $\underline{0.838}_{\pm0.000}$ & $0.808_{\pm0.021}$ \\
 
\Xhline{1pt}
\multirow{2}{*}{10x-hBC-D1} & AUC & $\mathbf{0.803}_{\pm0.017}$ & $0.488_{\pm0.011}$ & $0.485_{\pm0.111}$ & $0.276_{\pm0.072}$ & $0.377_{\pm0.005}$ & $0.530_{\pm0.172}$ & $0.380_{\pm0.003}$ & $0.469_{\pm0.007}$ & $0.503_{\pm0.000}$ & $\underline{0.649}_{\pm0.066}$ \\

& F1 & $\mathbf{0.698}_{\pm0.016}$ & $0.443_{\pm0.017}$ & $0.433_{\pm0.072}$ & $0.253_{\pm0.085}$ & $0.373_{\pm0.010}$ & $0.478_{\pm0.123}$ & $0.344_{\pm0.010}$ & $0.410_{\pm0.011}$ & $\underline{0.626}_{\pm0.000}$ & $0.465_{\pm0.158}$ \\

\Xhline{1pt}
\multirow{2}{*}{10x-hBC-E1} & AUC & $\mathbf{0.553}_{\pm0.046}$ & $\underline{0.536}_{\pm0.014}$ & $0.465_{\pm0.119}$ & $0.369_{\pm0.034}$ & $0.300_{\pm0.009}$ & $0.475_{\pm0.083}$ & $0.429_{\pm0.006}$ & $0.449_{\pm0.082}$ & $0.500_{\pm0.000}$ & $0.405_{\pm0.047}$ \\

& F1 & $\mathbf{0.739}_{\pm0.029}$ & $0.598_{\pm0.009}$ & $0.542_{\pm0.077}$ & $0.492_{\pm0.021}$ & $0.443_{\pm0.006}$ & $0.570_{\pm0.058}$ & $0.528_{\pm0.008}$ & $0.542_{\pm0.054}$ & $\underline{0.734}_{\pm0.000}$ & $0.081_{\pm0.095}$ \\ 

\Xhline{1pt}
\multirow{2}{*}{10x-hBC-F1} & AUC & $\mathbf{0.667}_{\pm0.009}$ & $0.485_{\pm0.046}$ & $0.476_{\pm0.017}$ & $0.493_{\pm0.011}$ & $0.483_{\pm0.005}$ & $0.477_{\pm0.074}$ & $0.379_{\pm0.004}$ & $0.380_{\pm0.074}$ & $\underline{0.500}_{\pm0.000}$ & $0.409_{\pm0.051}$ \\

& F1 & $\underline{0.858}_{\pm0.003}$ & $0.832_{\pm0.009}$ & $0.835_{\pm0.002}$ & $0.842_{\pm0.004}$ & $0.840_{\pm0.003}$ & $0.834_{\pm0.018}$ & $0.825_{\pm0.001}$ & $0.820_{\pm0.005}$ & $\mathbf{0.910}_{\pm0.000}$ & $0.036_{\pm0.022}$ \\ 

\Xhline{1pt}
\multirow{2}{*}{10x-hBC-G2} & AUC & $\mathbf{0.640}_{\pm0.079}$ & $0.524_{\pm0.016}$ & $0.482_{\pm0.074}$ & $0.457_{\pm0.016}$ & $0.430_{\pm0.008}$ & $\underline{0.576}_{\pm0.107}$ & $0.430_{\pm0.006}$ & $0.312_{\pm0.024}$ & $0.500_{\pm0.000}$ & $0.518_{\pm0.001}$ \\

& F1 & $\mathbf{0.544}_{\pm0.045}$ & $0.366_{\pm0.016}$ & $0.333_{\pm0.068}$ & $0.295_{\pm0.002}$ & $0.294_{\pm0.018}$ &  $0.434_{\pm0.095}$ & $0.273_{\pm0.006}$ & $0.214_{\pm0.029}$ & $\underline{0.510}_{\pm0.000}$ & $0.070_{\pm0.005}$ \\ 

\Xhline{1pt}
\multirow{2}{*}{10x-hBC-H1} & AUC & $\mathbf{0.732}_{\pm0.064}$ & $0.474_{\pm0.023}$ & $0.443_{\pm0.099}$ &  $\underline{0.625}_{\pm0.083}$ & $0.415_{\pm0.005}$ & $0.521_{\pm0.105}$ & $0.370_{\pm0.009}$ & $0.319_{\pm0.061}$ & $0.500_{\pm0.000}$ & $0.515_{\pm0.010}$ \\

& F1 & $\mathbf{0.516}_{\pm0.029}$ & $0.273_{\pm0.029}$ & $0.186_{\pm0.080}$ & $0.359_{\pm0.080}$ & $0.066_{\pm0.003}$ & $0.297_{\pm0.060}$ & $0.209_{\pm0.018}$ & $0.179_{\pm0.038}$ & $\underline{0.467}_{\pm0.000}$ & $0.418_{\pm0.113}$ \\ 

\Xhline{1pt}
\multirow{2}{*}{Mean} & AUC & $\mathbf{0.723}$ & $0.509$ & $0.494$ & $0.510$ & $0.433$ & $\underline{0.550}$ & $0.363$ & $0.371$ & $0.501$ & $0.542$ \\ 

& F1 & $\mathbf{0.741}$ & $0.542$ & $0.530$ & $0.552$ & $0.474$ & $0.577$ & $0.476$ & $0.474$ & $\underline{0.672}$ & $0.335$ \\ 
\Xhline{1.2pt}
\end{tabular}
}
\caption{Performance evaluation of anomalous tissue region detection across eight human breast cancer ST datasets. 
The table presents the results in terms of AUC and F1 scores, with each cell showing the average score from five independent runs and the corresponding standard deviation. The best score for each dataset is \textbf{bolded}, and the second-best score is \underline{underline}.}
\label{table:Main}
\end{table*}

\subsection{Latent Multimodal Reconstruction Loss-based Anomaly Discriminator (Stage III) }
Following Stage II, the original and reconstructed image patches of any spot $i$ are processed by a ResNet to generate their respective latent manifolds, denoted as $e_{img,i}\coloneqq \mathrm{ResNet}(\bm{P}_i)$ and $\tilde{e}_{img,i}\coloneqq \mathrm{ResNet}(\widetilde{\bm{P}}_i)$, respectively. Here, we employ a light-weight ResNet as the encoder since this stage focuses on calculating latent loss rather than for the more involved tissue image reconstruction task. Similarly, the manifolds of the original and reconstructed gene expression profiles of spot $i$ are generated by an MLP, denoted as $e_{gene,i}\coloneqq \mathrm{MLP}(\bm{x}_i)$ and $\tilde{e}_{gene,i}\coloneqq \mathrm{MLP}(\widetilde{\bm{x}}_i)$, respectively. Next, these manifolds are normalized, and a feed-forward network (FFN) maps their weighted averages to a latent space where the multimodal reconstruction error, $\ell_{rec,i}$, is calculated as follows:
\begin{equation}\label{original latent fusion}
    \bm{Z}_{fused,i}=\mathrm{FFN}\left(\beta\cdot \frac{e_{img,i}}{||e_{img,i}||}+(1-\beta)\cdot \frac{e_{gene,i}}{||e_{gene,i}||}\right) 
\end{equation}
\begin{equation}\label{reconstructed latent fusion}
     \widetilde{\bm{Z}}_{fused,i}=\mathrm{FFN}\left(\beta\cdot \frac{\tilde{e}_{img,i}}{||\tilde{e}_{img,i}||}+(1-\beta)\cdot \frac{\tilde{e}_{gene,i}}{||\tilde{e}_{gene,i}||}\right)
\end{equation}
\begin{equation}\label{reconst error}
    \ell_{rec,i}=\bm{Z}_{fused,i}-\widetilde{\bm{Z}}_{fused,i}
\end{equation}
where $0<\beta<1$ represents the relative weight assigned to the histology image. We then train a one-class classifier to collapse latent reconstruction errors of inliers into a compact hypersphere using the loss function:
\begin{equation}
\mathcal{L}_{occ} = \left\|\ell_{rec,i} - c \right\|^2
\end{equation}
where $c=\frac{1}{N}\sum\limits_{k=1}^{N} \ell_{rec,k}$. The training workflow of Stage III is also illustrated in Algorithm 2 of supplementary material C. At inference time, the anomaly score (AS) of a query spot $j$ is computed as the distance of its latent reconstruction loss to $c$:
\begin{equation}
    AS_j\coloneqq \left\|\ell_{rec,j} - c \right\|^2
\end{equation}
Given the observation that a gap exists between anomaly scores of inliers and anomalies (Figure 1 in supplementary material B), the AS threshold for discriminating inliers and anomalies is automatically determined using a Maximum A Posteriori-Expectation-Maximization (MAP-EM)-based mixture model, as detailed in supplementary material B.  

\begin{figure*}[t]
  \centering
  \includegraphics[width=0.8\textwidth]{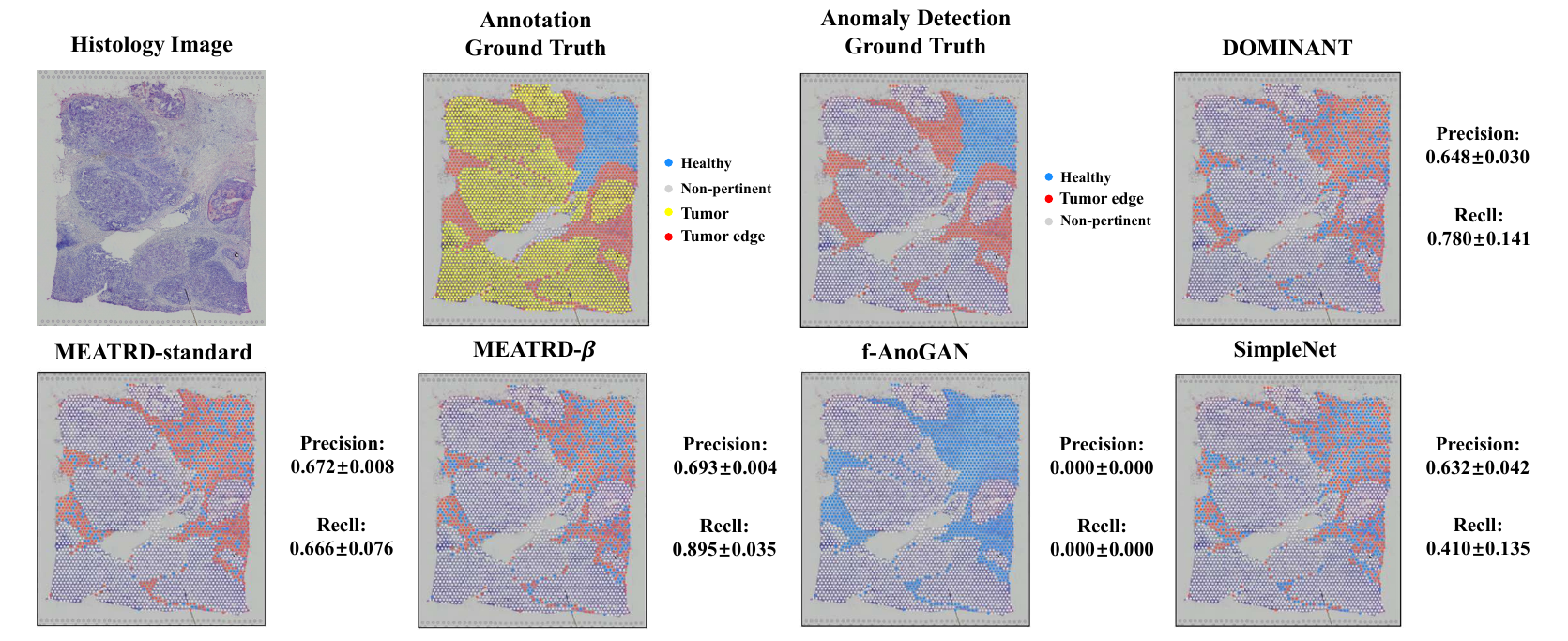}
  \caption{Visualized detection results of tumor edge regions that visually resemble the adjacent normal tissues in the 10x-hBC-I1 dataset. The first row, from left to right, displays the original histology image, the one annotated with ground truth region labels, the one highlighting the tumor edge region (in red) and the adjacent healthy region (in blue), and the one annotated with ATRs identified by DOMINANT. The second row presents images annotated with ATRs identified by their respective methods. 
The performance of each method is also quantified using mean precision and recall scores over five independent runs. These metrics, along with their standard deviations, are displayed right to each method's panel. }
\label{figr3}
\end{figure*}

\section{Experiments}
\subsection{Experimental Settings}

{\bfseries Datasets.}
MEATRD is extensively evaluated across eight breast cancer datasets and four primary sclerosing cholangitis (PSC) datasets. (see supplementary material E for data description and preprocessing).

\noindent {\bfseries Baselines.}
We select nine SOTA image-based, ST-based, and multi-modal AD methods as baselines. Image-based methods include two one-class classification methods, Patch SVDD \cite{yi2020patch} and SimpleNet \cite{liu2023simplenet}, alongside a reconstruction-based method, f-AnoGAN \cite{schlegl2019f}. For ST-based methods, we consider  scmap \cite{kiselev2018scmap},  a classification-based method, CAMLU \cite{CAMLU}, a reconstruction-based method, PREM \cite{pan2023prem}, a discriminative graph method, DOMINANT \cite{ding2019deep}, a generative graph method, and Spatial-ID \cite{shen2022spatial}, a classification-based method tailored for ST data. Additionally, M3DM \cite{wang2023multimodal} is chosen as a representative multimodal baseline.\\

\noindent {\bfseries Evaluation Protocols.}
AUC and F1 scores are used to evaluate the accuracy of ATR detection. For a fair comparison, the F1 score is calculated with the threshold matching the actual proportion of true anomalies \cite{ICL}. Reported metrics and standard deviations are averaged over five independent runs.

\subsection{Anomalous Tissue Region Detection}
In this experiment, as listed in supplementary material F, MEATRD is trained on eight human normal breast ST datasets (i.e., 10x-hNB-\{v03-v10\}) and tested on eight human breast cancer (i.e., 10x-hBC-\{A1-H1\}) ST datasets. \Cref{table:Main} showcases MEATRD's superiority over baselines in detecting ATRs across datasets, consistently ranking first in AUC scores and six times in F1 scores. It outperforms the second-best performing method with an average leap of 17.45\% in AUC scores and 10.31\% in F1 scores. Furthermore, Table 4 in supplementary material F indicates that our model performed well in detecting PSCs, demonstrating its generalization to other types of diseases. Generally, methods that use ST data, for example, DOMINANT, scmap and CAMLU, tend to outperform those that rely solely on histology images, indicating the pivotal role of gene expression information provided by ST data in aiding the detection of ATRs, especially those visually similar to normal tissues. Moreover, we find that, DOMINANT, a graph-based AD method, prevails over other baselines, and that M3DM, a multimodal method that utilizes both image and ST data yet fails to account for spatial relationships among spots, does not perform as well as MEATRD. These observations emphasize the value of spatial contextual information in accurate ATR detection.  

\subsection{Discovering Anomalous Tissue Regions Visually Similar to Normal Tissues}
To evaluate the efficacy of MEATRD in detecting ATRs with minimal visual distinctions from normal tissues, we conduct a comparative analysis on the 10x-hBC-I1 ST dataset, which encompasses a tumor edge region that visually blends with the adjacent normal tissues, as indicated in red in the annotated histology image from Figure \ref{figr3}. Our analysis includes: the standard MEATRD implementation (MEATRD-standard); MEATRD-$\beta$, a variant that downplays the influence of histology image by decreasing $\beta$ from 0.5 to 0.1 in \Cref {original latent fusion} and \Cref{reconstructed latent fusion}; DOMINANT, the top performing baseline utilizing ST data from the previous section; two leading image-based AD methods, f-AnoGAN and SimpleNet. The results, visually presented in Figure \ref{figr3}  demonstrate that MEATRD-$\beta$ more accurately identifies spots within the tumor edge region as anomalous, compared to the other competing methods. This finding is quantitatively supported by its highest precision (0.693) and recall (0.895) scores. The observation that MEATRD-standard, MEATRD-$\beta$, and DOMINANT prevail over the image-based AD methods underscores the value of using ST data for pinpointing pathogenic tissue regions that visually resemble normal tissues. Furthermore, DOMINANT's marginal performance edge over MEATRD-standard suggests that in this specific context, the histology image contributes very limited additional information. Indeed, MEATRD-$\beta$, which places greater emphasis on ST data, showcases an improved performance of 3.1\% in precision and 34.4\% in recall, compared to MEATRD-standard. Nonetheless, for scenarios involving low-quality ST data and visually traceable ATRs, incorporating visual cues from histology images are undoubtedly beneficial, as established in our prior analysis and ablation study.

\begin{table*}[htbp]
\vspace{-.2em}
\centering
%
\subfloat[
\textbf{Parameter $\alpha$}. A higher $\alpha$ in equation (10) indicates a higher gradient of image reconstruction loss.
\label{tab:decoder_depth}
]{
\centering
\begin{minipage}{0.29\linewidth}{\begin{center}
\tablestyle{4pt}{1.05}
\begin{tabular}{x{18}x{24}x{24}}
case & AUC & F1 \\
\shline
0.9 & 0.678 & 0.696 \\
0.5 & \baseline{\textbf{0.723}} &\baseline{\textbf{0.741}} \\
0.1 & 0.709 & 0.725 \\
\end{tabular}
\end{center}}\end{minipage}
}
\hspace{2em}
%
\subfloat[
\textbf{Parameter $\beta$}. A higher $\beta$ in equation (12) and (13) means the image plays a larger role in the final decision.
]{
\begin{minipage}{0.29\linewidth}{\begin{center}
\tablestyle{4pt}{1.05}
\begin{tabular}{x{18}x{24}x{24}}
case & AUC & F1 \\
\shline
0.9 &0.654 	&0.668   \\
0.5 &\baseline{\textbf{0.723}} 	&\baseline{\textbf{0.741}}\\ 
0.1 &0.699	 	&0.718    \\ 
\end{tabular}
\end{center}}\end{minipage}
}
\hspace{2em}
%
\subfloat[
\textbf{Embedding dimension}. The dimension of visual embedding $\mathcal{\bm{Z}}_{img,i}^{(l)}$ and gene embedding $\mathcal{\bm{Z}}_{gene,i}^{(l)}$ in equation (5). 
]{
\begin{minipage}{0.29\linewidth}{\begin{center}
\tablestyle{4pt}{1.05}
\begin{tabular}{x{18}x{24}x{24}}
dim & AUC & F1 \\
\shline
128	&0.705	&0.726 \\
256	&\baseline{\textbf{0.723}}	&\baseline{\textbf{0.741}} \\
512	&0.721	&0.735      \\ 
\end{tabular}
\end{center}}\end{minipage}
}
\hspace{2em}
%
\subfloat[
\textbf{Bottleneck dimension}. The dimension of fused bottleneck embeddings $\mathcal{\bm{Z}}_{fb,i}^{(l)}$ in equation (5).
]{
\centering
\begin{minipage}{0.2\linewidth}{\begin{center}
\tablestyle{4pt}{1.05}
\begin{tabular}{x{18}x{24}x{24}}
dim & AUC & F1 \\
\shline
16	&\baseline{\textbf{0.723}} 	&\baseline{\textbf{0.741}} \\
64	&0.715 		&0.728 \\
256	&0.682 	&0.711 \\
\end{tabular}
\end{center}}\end{minipage}
}
\hspace{2em}
%
\subfloat[
\textbf{Detection dimension}. The dimension of the latent reconstruction error $\ell_{rec,i}$ in equation (14). 
]{
\begin{minipage}{0.2\linewidth}{\begin{center}
\tablestyle{4pt}{1.05}
\begin{tabular}{x{18}x{24}x{24}}
dim & AUC & F1 \\
\shline
64 &0.606 	&0.623 \\ 
128 &0.720 	&0.732 \\ 
256 &\baseline{\textbf{0.723}} 	&\baseline{\textbf{0.741}}    \\ 
\end{tabular}
\end{center}}\end{minipage}
}
\hspace{2em}
%
\subfloat[
\textbf{MGDAT Layers}. Three-layer MGDAT blocks are
effective.
]{
\begin{minipage}{0.2\linewidth}{\begin{center}
\tablestyle{4pt}{1.05}
\begin{tabular}{x{18}x{24}x{24}}
blocks & AUC & F1 \\
\shline
2	&0.694	&0.719 \\
3	&\baseline{\textbf{0.723}}	&\baseline{\textbf{0.741}}\\
4	&0.533	&0.565    \\ 
\end{tabular}
\end{center}}\end{minipage}
}
\hspace{2em}
%
\subfloat[
\textbf{GAT Attn Heads}. Using two-head attention in MGDAT blocks is more reliable.
]{
\begin{minipage}{0.22\linewidth}{\begin{center}
\tablestyle{4pt}{1.05}
\begin{tabular}{x{18}x{24}x{24}}
case & AUC & F1 \\
\shline
1	&0.718	&0.730 \\
2	&\baseline{\textbf{0.723}}	&\baseline{\textbf{0.741}}\\
4	&0.721	&0.737    \\ 
\end{tabular}
\end{center}}\end{minipage}
}
\hspace{2em}
\vspace{-.1em}
\caption{Sensitivity analysis of hyperparameter in MEATRD across eight human breast cancer datasets. Default settings are marked in \colorbox{baselinecolor}{gray}.}
\label{table:hyper} \vspace{-.5em}
\end{table*}

\subsection{Ablation Studies}
We conduct ablation studies over the eight human breast cancer ST datasets (i.e., 10x-hBC-\{A1-H1\}) to investigate the effects of MEATRD's key components in ATR detection. These components include using multiple data modality,  multimodal data fusion using fused bottleneck embedding, masking for target node reconstruction, multimodal reconstruction losses in the one-class classifier in Stage III, enlarging anomaly score discrepancy between inliers and anomalies using a one-class classifier, using Mobile-Unet as the pretraining backbone in Stage I. The descriptions detailed in the \textit{Ablation Studies} section in supplementary material F, demonstrate that removing any of these components leads to suboptimal performance. This is due to the inefficient use of cross-modal complementary information, less effective addressing of model over-generalization, and increased sensitivity to reference-target domain shifts. 

\begin{table}[tbp]
\centering
\fontsize{6pt}{6pt}\selectfont
\setlength{\tabcolsep}{2.6pt}
\scriptsize
\begin{tabular*}{\linewidth}{c|cccccc|c}
\Xhline{1.2pt}
\multirow{2}{*}{\centering Metric} &\multicolumn{7}{c}{\small Ablation study}\\
\cline{2-8}
 & ST only & Image only & w/o MGDAT & w/o TNM & w/o RE & w/o OC & Full\\ 
\hline
AUC   & 0.631  & 0.497  & 0.639 & 0.655  & 0.642  & 0.584 & \textbf{0.723} \\
F1    & 0.667  & 0.544  & 0.689  & 0.699  & 0.685 & 0.631 & \textbf{0.741} \\
\Xhline{1.2pt}
\end{tabular*}
\caption{Ablation study of key components in MEATRD across eight human breast cancer datasets. Method performance is gauged through average AUC and F1 scores. ''Full'' represents the complete MEATRD model. ''ST Only'' and ''Image Only'' utilize only ST data or histology images, respectively. ''w/o MGDAT'' omits the MGDAT block. ''w/o TNM'' omits the target-node-masking technique. ''w/o RE'' substitutes the latent multimodal reconstruction errors with direct spot embeddings for input to the discriminative model in Stage III. ''w/o OC'' discards the entire stage III and utilizes spot reconstruction errors as anomaly scores for ATR detection.}
\label{table:ablation}
\end{table}

\subsection{Sensitivity Analysis}
Here, we conduct sensitivity analyses on eight 10x-hBC datasets to examine the effects of MEATRD’s key hyperparameters, including $\alpha\ \text{and}\ \beta$, which control the relative weights between gene and image modalities in Stage II and III; the dimensions of visual and gene embedding from Stage I, bottleneck embedding in Stage II, and the inputs to the one-classification classifier in Stage III; the number of MGDAT layers and its attention heads. The effect of these parameters on MEATRD's performance, measured by AUC and F1 scores, is presented in \Cref{table:hyper}. Detailed results are provided in supplementary material F.3.

\subsection{Complexity Analysis}
We analyze the model complexity of MEATRD across its three stages by evaluating the number of parameters, computational performance (MFlops), time complexity, training time, and inference time. We also compare these metrics with the nine baseline methods. Detailed results are provided in supplementary material F.4. In summary, MEATRD is scalable to the number of spots and edges (proportional to the number of spots due to the adjacency matrix setting) and demonstrates good efficiency in our experiments.

\section{Conclusion}
In this paper, we propose MEATRD, a pilot method that integrates histology images and ST data to enhance ATR detection at both visual and molecular levels. MEATRD treats tissue spots as nodes in an attributed graph to embed spatial relationships into their representations. The MGDAT network, a key innovation of MEATRD, facilitates effective cross-node and cross-modality information exchange, enabling comprehensive graph representation learning. MEATRD harmonizes one-class classification with reconstruction deviation-based AD detection, simultaneously addressing the challenges of reference-target domain shift and model over-generalization. Rigorous evaluations on a suite of real ST datasets have demonstrated MEATRD's superiority over various SOTA AD methods in detecting ATRs including those that are visually akin to contextual normal tissues. Furthermore, MEATRD also offers a framework generalizable to other multimodal AD tasks involving compatible imagery and graph data modalities. 

\section{Acknowledgments}
The project is funded by the Excellent Young Scientist Fund
of Wuhan City (Grant No. 21129040740) to X.S.

\bibliography{aaai25}

\appendix

\clearpage
\begin{center}
  \textbf{\LARGE Supplementary Material}
\end{center}

\section{Related Work}
\subsection{Localized Anomaly Detection in Image}
Related works in this field can be broadly divided into two categories: one-class classification methods and reconstruction-based methods. The former aims to delineate normal data distributions and boundaries in a latent space at training time, labeling instances occurring in low-probability density regions (i.e., falling outside the boundary) as anomalies at test time \cite{shvetsova2021anomaly}. For example, Patch SVDD \cite{yi2020patch} assesses anomalies according to their proximity to the nearest inlier in a latent space that is learned by minimizing distances between nearby inliers. Another example, SimpleNet \cite{liu2023simplenet} creates pseudo-anomalies by introducing random noises to extracted visual features of inliers, and trains a separating hyperplane-based discriminator for anomaly differentiation. The main limitation of these methods is their dependency on effective representation learning \cite{sohn2020learning}, which may be compromised by batch effects between reference and target datasets \cite{ouardini2019towards}. 

On the other hand, reconstruction-based methods, trained on normal data only, posit that inliers can be reconstructed more faithfully from their latent manifolds than anomalies. For instance, f-AnoGan \cite{schlegl2019f}, a WGAN\cite{arjovsky2017wasserstein}-based method for AD in medical images, employs a discriminator-guided encoder to obtain reconstruction residuals as anomaly scores. While theoretically more robust to batch effects since only anomalies are identified based on reconstruction errors within the same batch, these methods may suffer from model over-generalization, leading to minor reconstruction errors for anomalies \cite{liu2023simplenet,ristea2022self}. Overall, methods for localized AD in images often overlook the contextual surroundings \cite{sabour2017dynamic,ristea2022self}, although some, such as SSPCAB \cite{ristea2022self} and PatchCore \cite{roth2022towards}, attempt to aggregate information from neighboring patches through simplified means such as adaptive averaging pooling. In contrast, our method, by virtue of the MGDAT network, can comprehensively harness contextual information for improved AD.  \cite{he2015convolutional}

\subsection{Anomaly Detection using Gene Expression Data}
Tissue spots in ST closely resemble single cells in single-cell RNA sequencing (scRNA-seq), augmented with spatial location information. This similarity offers an opportunity to apply anomalous cell (AC) detection methods to ATR detection in ST. Traditional AC detection methods treat scRNA-seq data as tabular, identifying ACs through cell type classification tasks. For example, scmap \cite{kiselev2018scmap} computes gene expression similarities between query cells and centroids of known cell types, designating those below a threshold as anomalies. Such classification-based methods depend heavily on labeled references, often a scarce and costly resource. To circumvent this limitation, CAMLU \cite{CAMLU}, a reconstruction-based method utilizing unlabeled reference data only, identifies ACs in the target dataset using informative genes selected as per their reconstruction deviations. However, these methods neglect spatial information inherent to ST data, which is crucial for accurate ATR detection. To bridge this gap, specialized methods have been developed, typically leveraging graph neural networks (GNN) to incorporate spatial relationships among spots \cite{hu2021spagcn,dong2022deciphering}. Among these, to the best of our knowledge, Spatial-ID \cite{shen2022spatial} is currently the sole method capable of identifying ATRs by utilizing a classifier, pre-trained on labeled scRNA-seq data, to classify spots based on their latent manifolds generated via a variational graph autoencoder. Spots with uncertain soft assignments are labeled as anomalies. However, Spatial-ID, like many other gene-oriented AD methods, is prone to high false positive rates due to its reliance on assignment uncertainties, often arising from inlier similarities rather than genuine anomalies \cite{CAMLU}. 

An alternative strategy, bypassing the classification framework, involves modeling ST data as an attributed graph and applying node-level graph anomaly detection (GAD) methods, which can be generative or discriminative \cite{pan2023prem}. For instance, PREM \cite{pan2023prem} determines anomalous nodes based on their anomaly scores calculated as the dissimilarity between ego and neighbor node embeddings, which are generated through graph contrastive learning. DOMINANT \cite{ding2019deep}, a generative GAD method, leverages a graph convolutional network (GCN) \cite{kipf2016semi} to reconstruct both nodal attributes and topological structure, using combined reconstruction errors as anomaly scores. Generally, all methods discussed in this section are limited by their heavy dependence on the quality of ST data and falling short of exploiting visual information available in histology images to improve the accuracy of ATR detection.

\subsection{Multimodal Anomaly Detection}
By far, the development of multimodal AD methods has been predominantly focused on industrial AD scenarios involving the simultaneous use of 2D and 3D data. Recent methods in this field include M3DM \cite{wang2023multimodal} and AST \cite{rudolph2023asymmetric}. M3DM adopts a contrastive learning-based approach to fuse manifolds of segmented patches from 3D point clouds and RGB images, based on which a discriminative model is trained for anomaly decision. AST concatenates features extracted from RGB images and 3D depth maps, serving as inputs to asymmetric student and teacher networks that determine anomaly scores as per their output discrepancies. However, there is a significant gap in developing multimodal ATR detection methods that combine gene expression data and histology images.

\section{Determining anomaly score threshold}

Based on the observation that there is a gap between anomaly scores of inliers and true anomalies, as shown in \Cref{suppfig1}, we designed a two-component mixture model to automatically determine the anomaly score threshold that discriminate inliers and anomalies. Specifically, the distribution of anomaly scores is modeled as a univariate Gaussian Mixture Model (GMM) with two components corresponding to anomalous and normal instances, respectively. We specify the prior for anomaly abundance as a beta distribution and the priors for the mean and variance of the two Gaussian components as a Normal Inverse Chi-squared (NIX) distribution. The parameters of these priors are estimated based on inlier anomaly scores in the reference dataset. Utilizing the Maximum A Posteriori (MAP)-EM algorithm, we infer the parameters for both Gaussian components and then assign spots into either normal or anomalous groups based on their probabilities within each component. Specifically, let  $\Theta=\left\{\pi,\mu_k,\sigma_k^2,\forall k\in\left\{1,2\right\}\right\}$ represent the GMM parameters, where $\pi\in\left[0,1\right]$ represents the proportion of anomalies, and $\mu_k,\sigma_k^2$ represent the mean and variance for the $k$-th component, respectively, with the constraint that $\mu_1>\mu_2$. Then, the probability density function of an anomaly score $d_i$ can be formulated as:
\begin{equation}  P\left(d_i\middle|\Theta\right)=\pi\mathcal{N}\left(d_i\middle|\mu_1,\sigma_1^2\right)+(1-\pi)\mathcal{N}\left(d_i\middle|\mu_2,\sigma_2^2\right)
\end{equation}
\begin{equation}\pi \sim \mathrm{Beta}\left(\pi\middle|a,b\right)\end{equation}
\begin{equation}
\mu_k,\sigma_k^2\sim \mathrm{NIX}\left(\mu_k,\Sigma_k\middle|m_0,\kappa_0,s_0^2,\nu_0\right)
\end{equation}
Parameters for the priors in the GMM are empirically set based on the reference dataset’s anomaly scores $\delta_{i}, {\forall} i\in \{1,2,\cdots, N_{ref}\}$ :
\begin{small} \begin{equation}
m_0 = \frac{\sum_{i=1}^{N_{ref}}\delta_{i}}{N_{ref}},\  \kappa_0=0.01,\ \nu_0=3,\ s_0^2=\frac{\sum_{i=1}^{N_{ref}}\left(\delta_i-m_0\right)}{N_{ref}}
\end{equation} \end{small}
\begin{equation}
a=1,\ b=10
\end{equation}

The values of $a$ and $b$ can be adjusted if prior knowledge about anomaly abundance is available. The complete data log likelihood for the posterior, denoted as $\ell_c\left(\Theta\right)$, is expressed as:
\begin{equation} 
\begin{aligned}
 \ell_{c}\left(\Theta\right) &= \mathrm{log}P\left(\mathcal{D}\ \middle|\Theta\right)\\ 
 &= \sum_{i}\Big[\mathbb{I}\left(z_i=1\right)\left(\mathrm{log}\pi+\mathrm{log}\mathcal{N}\left(d_i\middle|\mu_1,\sigma_1^2\right)\right)\\ 
 &+ \mathbb{I}\left(z_i=2\right)\left(\mathrm{log}(1-\pi)+\mathrm{log}\mathcal{N}\left(d_i\middle|\mu_2,\sigma_2^2\right)\right)\Big]\\ &+\mathrm{log}\mathrm{Beta}\left(\pi\middle|a,b\right) \\ 
 &+ \sum_{k=1}^{2}{\mathrm{log}\mathrm{NIX}\left(\mu_k,\sigma_k^2\middle|m_0,\kappa_0,s_0^2,\nu_0\right)}
\end{aligned} 
\end{equation}

Here, $z_i$ denotes the component membership of spot $i$. In the $t$-th iteration of the E-step, the expected sufficient statistics ${\overline{z_i}}^{(t)}$ is derived from $\Theta^{(t-1)}$. In the subsequent M-step, $\Theta^{(t-1)}$ is updated to $\Theta^{(t)}$ by maximizing the auxiliary function $Q\left(\Theta,\Theta^{(t-1)}\right)=E\left({\ell}_{c} \left(\Theta\right)\big|\Theta^{(t-1)}\right)$. We elaborate our MAP-EM algorithm below:

\paragraph{MAP-EM inference of GMM parameters.}
We first list the mathematical notations used in the inference below:
\begin{table}[H]
    \centering
    \renewcommand{\arraystretch}{1.1}
    \fontsize{5.85pt}{7.5pt}\selectfont
    \resizebox{\linewidth}{!}{
    \begin{tabular}{ll} 
\Xhline{0.8pt}
       \textbf{Notation}   & \textbf{Description} \\ \Xhline{0.5pt}
        $\mathcal{D}\ =d_i, \forall i \in \{1,2,\cdots,N\}$ & Set of anomaly scores of target spots.\\
        $\Delta\ =\left\{\delta_i, \forall i \in \{1,2,\cdots,N_{ref}\} \right\}$ & Set of anomaly scores of reference spots.\\
        $N$ & Number of target spots.\\
        $N_{ref}$ & Number of reference spots.\\
        $\pi_1$ & Anomaly abundance among the target spots.\\
        $\pi_2=1-\pi_1$ & Abundance of normal spots among the target spots.\\
        $\Theta=\left\{\pi_k,\mu_k,\sigma_k^2,\forall k \in \left\{1,2\right\}\right\}$ & Parameters of the $k$-th GMM components.\\
        $z_i\in\left\{1,2\right\}$ & GMM component membership of the spot $i$.\\  
\Xhline{0.8pt}
    \end{tabular}
}
\caption{Overview of notations in MAP-EM inference.}
\end{table} 

Initially, we introduce a prior on $\pi_1$ as a Beta distribution, and a conjugate joint prior on $\mu_k$,$\sigma_k^2$ as a normal inverse chi-squared (NIX) distribution:
\begin{equation}
\pi_1 \sim \mathrm{Beta}\left(\pi\middle|a,b\right)
\end{equation}
\begin{equation}
\begin{aligned}
\mu_k,\sigma_k^2 &\sim \mathrm{NIX}\left(\mu_k,\sigma_k^2\middle|m_0,\kappa_0,s_0^2,\nu_0\right)\\ &=\mathcal{N}\left(\mu_k\middle|m_0,\sigma_k^2/\kappa_0\right)\chi^{-2}\left(\sigma_k^2\middle|s_0^2,\nu_0\right)
\end{aligned}
\end{equation}

Here, we set the parameters of the prior distributions based on the anomaly scores of spots in the reference dataset:
\begin{small} \begin{equation}
m_0=\frac{\sum_{i=1}^{N_{ref}}\delta_i}{N_{ref}}, \kappa_0=0.01, \nu_0=3, s_0^2=\frac{\sum_{i=1}^{N_{ref}}\left(\delta_i-m_0\right)}{N_{ref}}
\end{equation} \end{small}
\begin{equation}
a=1, b=10
\end{equation}

Note that the values of a\ and\ b can be set to more appropriate values if prior knowledge about the abundance of anomalies is available. The posterior complete data log likelihood can be written as:
\begin{equation} 
\begin{aligned}
\ell_c\left(\Theta\right) & = \mathrm{log}P\left(\mathcal{D}\ \middle|\Theta\right)\\ 
& = \sum_{i}\sum_{k} \mathbb{I} \left(z_i=k\right) \left(\mathrm{log}\pi_k+\mathrm{log}\mathcal{N}\left(d_i\middle|\mu_k,\sigma_k^2\right)\right) \\ &+\mathrm{log}\mathrm{Beta}\left(\pi\middle|a,b\right)\\ 
&+\sum_{k=1}^{2}{\mathrm{log}\mathrm{NIX}\left(\mu_k,\sigma_k^2\middle|m_0,\kappa_0,s_0^2,\nu_0\right)}
\end{aligned} 
\end{equation}

\textbf{E step}. In the $t$-th iteration, we have the auxiliary function $Q$ as:
\begin{equation}
\begin{aligned}
  & Q\left(\Theta, \Theta^{(t-1)}\right) = \mathbb{E}\left[\ell_c(\Theta)\big|\Theta^{(t-1)}\right] \\
  & = \sum_i\sum_{k=1}^2 P\left(z_i=k\big|d_i, \Theta^{(t-1)}\right) \\
  &\qquad\left[\mathrm{log}\pi_k^{(t-1)}+\mathbb{E}\left(\mathrm{log}N\left(d_i\big|\mu_k^{(t-1)},{(\sigma_k^2)}^{(t-1)}\right)\right)\right] \\
  &+ \mathrm{log}\mathrm{Beta}\left(\pi|a,b\right) + \sum_{k=1}^2\mathrm{log}\mathrm{NIX}\left(\mu_k, \sigma_k^2\big|m_0, \kappa_0, s_0^2, \nu_0\right)
\end{aligned}
\nonumber 
\end{equation}

The expected sufficient statistics (ESS) are:
\begin{equation}
\begin{aligned}
  \overline{z_{i,k}} &= P\left(z_i=k\big|d_i,\Theta^{(t-1)}\right) \\ 
  &=\frac{\pi_k^{(t-1)}\mathcal{N}\left(d_i\big|\mu_k^{(t-1)},(\sigma_k^2)^{(t-1)}\right)}{\sum_{k^\prime}\pi_{k^\prime}^{(t-1)}\mathcal{N}\left(d_i\big|\mu_{k^\prime}^{(t-1)},(\sigma_{k^\prime}^2)^{(t-1)}\right)}
\end{aligned}
\end{equation}

\textbf{M step}. In the $t$-th iteration, the expected complete posterior data log likelihood is:
\begin{equation} \begin{split}
& Q\left(\Theta,\Theta^{\left(t-1\right)}\right) \propto \\ & \sum_{k=1}^{2}\sum_{i}\left[{\overline{z_{i,k}}}^{\left(t\right)}\left(\mathrm{log}\pi_k-\frac{\mathrm{log}\left(\sigma_k^2\right)}{2}-\frac{\left(d_i-\mu_k\right)^2}{2\sigma_k^2}\right)\right]\\ 
&\quad+\mathrm{log}\mathrm{Beta}\left(\pi\middle|a,b\right) \\ &\quad+\sum_{k=1}^{2}\left[\mathrm{log}\mathcal{N}\left(\mu_k\middle|m_0,\sigma_k^2/\kappa_0\right)+\mathrm{log}\chi^{-2}\left(\sigma_k^2\middle|s_0^2,\nu_0\right)\right]
\end{split} \end{equation}

We maximize $Q\left(\Theta,\Theta^{\left(t-1\right)}\right)$ with respect to $\Theta$. The posterior distribution of $\pi_1\ $ and $\left\{\mu_k,\sigma_k^2\right\}$ are:
\begin{equation}
\pi_1 \sim \mathrm{Beta}\left(\pi\middle|a^{\left(t\right)},b^{\left(t\right)}\right)
\end{equation}
\begin{equation}
a^{\left(t\right)}=a+\sum_{i}{\overline{z_{i,1}}}^{\left(t\right)}
\end{equation}
\begin{equation}
b^{\left(t\right)}=b+\sum_{i}
{\overline{z_{i,2}}}^{\left(t\right)}
\end{equation}
\begin{equation}
\mu_k,\sigma_k^2 \sim \mathrm{NIX}\left(\mu_k,\sigma_k^2\middle|m_k^{\left(t\right)},\kappa_k^{\left(t\right)},\left(s_k^2\right)^{\left(t\right)},\nu_k^{\left(t\right)}\right)
\end{equation}
\begin{equation}
{\overline{z_k}}^{\left(t\right)}=\sum_{i}{\overline{z_{i,k}}}^{\left(t\right)}
\end{equation}
\begin{equation}
{\bar{d_k}}^{\left(t\right)}=\frac{\sum_{i}{({\overline{z_{i,k}}}^{\left(t\right)}}d_i)}{{\overline{z_k}}^{\left(t\right)}}
\end{equation}
\begin{equation}
\nu_k^{\left(t\right)}=\nu_0+{\overline{z_k}}^{\left(t\right)},\ \kappa_k^{\left(t\right)}=\kappa_0+{\overline{z_k}}^{\left(t\right)}
\end{equation}
\begin{equation}
m_k^{\left(t\right)}=\frac{{\overline{z_k}}^{\left(t\right)}{\bar{d_k}}^{\left(t\right)}+m_0\kappa_0}{\kappa_k^{\left(t\right)}}
\end{equation}
\begin{equation}
\left(s_k^2\right)^{\left(t\right)}={\nu_0s}_0^2+\sum_{i}{({\overline{z_{i,k}}}^{\left(t\right)}}d_i^2)+\kappa_0m_0^2-{\overline{z_k}}^{\left(t\right)}\left(m_k^{\left(t\right)}\right)^2
\nonumber 
\end{equation}
Then we have the MAP estimates of $\pi_1$, $\mu_k\ $ and $\sigma_k^2$ as $\pi_1^{(t)}$,$\mu_k^{\left(t\right)}$ and $\left(\sigma_k^2\right)^{\left(t\right)}$:
\begin{equation}
\pi_1^{\left(t\right)}=\frac{a^{\left(t\right)}-1}{a^{\left(t\right)}+b^{\left(t\right)}-2}
\end{equation}
\begin{equation}
\mu_k^{\left(t\right)}=m_k^{\left(t\right)}
\end{equation}
\begin{equation}
\left(\sigma_k^2\right)^{\left(t\right)}=\frac{{\nu_k^{\left(t\right)}\left(s_k^2\right)}^{\left(t\right)}}{\nu_k^{\left(t\right)}+3}
\end{equation}

Next, the EM algorithm continues to E step of the $\left(t+1\right)$-th iteration to update ${\overline{z_{i,k}}}^{(t+1)}$ ,$\forall i\in\left[1,N\right]$,$\forall k\in\left\{1,2\right\}$ until either convergence is achieved, or a pre-specified number of iterations is reached. Finally, the soft assignment of spot $i$ to the anomalous group $(\mathcal{Q}_{i,1})$ is calculated by plugin $\Theta$:
\begin{equation}
q_{i,1}=\pi_1\mathcal{N}\left(d_i\middle|\mu_1,\sigma_1^2\right)
\end{equation}
\begin{equation}
\mathcal{Q}_{i,1}=\frac{q_{i,1}}{\sum_{k} q_{i,k}},\forall i\in\{1,2,\cdots,N\},\forall k\in\left\{1,2\right\}
\end{equation}
If  $\mathcal{Q}_{i,1}>0.5$, then spot $i$ is determined as an anomaly. 

\section{Algorithm for \textit{MEATRD}}

\renewcommand{\algorithmicrequire}{\textbf{Input:}}
\renewcommand{\algorithmicensure}{\textbf{Output:}}
{
\begin{center}
  \begin{tabular}{p{\textwidth}}
    \begin{algorithm}[H]
      \caption{Stage II training.}
      \begin{algorithmic}[1]
        \Require
      Gene expression profiles $\bm{X} \in \mathbb{R}^{N\times G}$; Image patches $\bm{P} \in \mathbb{R}^{N\times h\times w\times c}$; Attributed graph $G(V,A,\mathcal{\bm{Z}})$; Number of nodes \textit{N}; Parameter of Image modality $\lambda$;
      Parameter of Gene modality $\alpha$.
    \vspace{0pt}
    \renewcommand{\algorithmicrequire}{\textbf{Definition:}}
    \Require
    Pre-trained Mobile-UNet encoder $E_1$; Pre-trained Mobile-Unet decoder $D_1$; Gene encoder $f_E$; Gene dncoder $f_D$; MGDAT network $\mathcal{F}$; ResNET-based image decoder $D_2$; GNN-based gene decoder $D_3$; Feed-forward network $f$; L1 reconstruction loss function $\mathcal{L}_{1}$; SSIM loss function $\mathcal{L}_{SSIM}$; SCE loss fustion $\mathcal{L}_{SCE}$.
    \vspace{0pt}
    \Ensure
    Reconstructed ST data of query spot $\hat{\mathbf{X}}_b$; Reconstructed histology image of query spot $\hat{\mathbf{P}}_b$.
    \vspace{0pt}
    
    \For {$\mathbf{X}_b$, $\mathbf{P}_b$ in $\bm{X}$, $\bm{P}$}
    \Comment{Processing \textit{Stage II}} 
        \State 
        $Z_{gene} = f_E(\mathbf{X}_b)$, $Z_{img} = E_1(\mathbf{P}_b)$.
        \vspace{0pt}
        \State 
        $Z_{gene}, Z_{img} = \mathcal{F}(Z_{gene}, Z_{img})$
        \vspace{0pt}
        \State 
        $\mathbf{P}_b = D_2(Z_{img})$ , $\mathbf{X}_b = D_3(Z_{gene})$ .
        \vspace{0pt}
        \State $\mathcal{L}_{rec}=\mathcal{L}_{ssim}(\mathbf{P}_b, \hat{\mathbf{P}}_b)+\lambda \mathcal{L}_1(\mathbf{P}_b, \hat{\mathbf{P}}_b)+\alpha \mathcal{L}_{SCE}(\mathbf{X}_b, \hat{\mathbf{X}}_b)$.
        \vspace{0pt}
        \State Update parameters of $E_1, f_E, \mathcal{F}, D_2, D_3$ using $\mathcal{L}_{rec}$.
        \EndFor 
    \vspace{0pt} \\
   
    \Return $\hat{\mathbf{P}}_b$, $\hat{\mathbf{X}}_b$
      \end{algorithmic}
    \end{algorithm}\\
  \end{tabular}
  \label{tab:training II}
\end{center}
}
\vspace{-1cm}

\renewcommand{\algorithmicrequire}{\textbf{Input:}}
\renewcommand{\algorithmicensure}{\textbf{Output:}}
{
\begin{center}
  \begin{tabular}{p{\textwidth}}
    \begin{algorithm}[H]
      \caption{Stage III Training.}
      \begin{algorithmic}[1]
        \Require  
      Gene expression profiles $\bm{X} \in \mathbb{R}^{N\times G}$; Image patches $\bm{P} \in \mathbb{R}^{N\times H\times W\times C}$; Maximum epochs $E_{max}$.
    \renewcommand{\algorithmicrequire}{\textbf{Definition:}}
    \Require
    Feed-forward network $f$; Image encoder in stage III $E_2$; Gene encoder in stage III $E_3$; Reconstruction error $\ell_{rec,b}$; Dimension of hyperspherical space $D$
    \Ensure
      SVDD Center $c\in \mathbb{R}^{N\times D}$. \\
    \vspace{0pt} 
    $\hat{\mathbf{P}}$, $\hat{\mathbf{X}}$ =  Learning of Spot Reconstruction ($\mathbf{P}$, $\mathbf{X}$) \Comment{Processing \textit{Stage III}} 
    \State Initialize $E_2$, $E_3$, and $f$. 
    \While{$epoch < E_{max}$}
        \State Compute the center $c$.
        \vspace{0pt}
        \For {$\mathbf{P}_b$, $\mathbf{X}_b$, $\hat{\mathbf{P}}_b$, $\hat{\mathbf{X}}_b$ in $\bm{P}$, $\bm{X}$, $\hat{\bm{P}}$, $\hat{\bm{X}}$}
            \State {\small $Z_{fused} = f(\beta norm(E_2(\mathbf{P}_b))+norm(E_3(\mathbf{X}_b))))$}.
            \vspace{0pt}
            \State {\small $\hat{Z}_{fused} = f(\beta norm(E_2(\hat{\mathbf{P}}_b))+norm(E_3(\hat{\mathbf{X}}_b))))$}.
            \vspace{0pt}
            \State $\ell_{rec,d} = Z_{fused} -\hat{Z}_{fused}$. 
            \vspace{0pt}
            \State$\mathcal{L}_{SVDD}=\left\| \ell_{rec,b} -c\right\|^2$.
            \State Update parameters of $E_2, E_3, f$ using $\mathcal{L}_{SVDD}$
        \EndFor
   \EndWhile \\
    \Return $c$
      \end{algorithmic}
    \end{algorithm} \\
  \end{tabular}
  \label{tab:training III}
\end{center}
}
\vspace{-1cm}

\section{Theoretical Analysis}
\subsection{Fused Bottleneck Encoding as a Minimally Sufficient Representation of Modality-Specific, Task-Relevant Information }
In this section, we begin with the mathematical notations (\Cref{tab:notation}), properties (\Cref{properties 1}), definitions (\Cref{Definition 1.1}), and assumptions (\Cref{Assumption 1.1} and \Cref{Assumption 1.2}) pertinent to our theoretical analysis. We then prove that the fused bottleneck encoding serves as a sufficient statistic \cite{tian2020makes} for capturing complementary task-relevant information across data modalities (\Cref{Proposition inclusiveness}), as illustrated in Supplementary Figure \ref{fig_com}. Finally, we prove that the fused bottleneck encoding is the most informationally compact among all sufficient encodings (\Cref{Proposition compactness}). 

\begin{table}[H]
    \centering
    
    \fontsize{7.8pt}{10pt}\selectfont
    \resizebox{\linewidth}{!}{
    \begin{tabular}{ll} 
\Xhline{1pt}
       \textbf{Notation}   & \textbf{Description} \\ \hline
       \Xhline{0.3pt}
        $v_i$&The view associated with the $i$-th data modality.\\
        $b_0$ &The biological contents shared between data modalities.\\ 
        $b_i$ &The biological contents specific to the $i$-th data modality.\\
        $I(*)$ &The information set inherent to *.\\
        $M$&The mutual information function.\\
        $H$&The entropy function.\\
        $f_1$ and $z_1$&The encoder and encoding for view $v_1$.\\
 $f_2$ and $z_2$&The encoder and encoding for view $v_2$.\\
 $f_3$ and $z_3$&The fusion bottleneck encoder and fused encoding.\\ 
\Xhline{1pt}
    \end{tabular}
}
\caption{Summary of notation.}\label{tab:notation}
\end{table}

\begin{definition}\label{Definition 1.1} 
Information function $I(x)$ denotes the information set inherent in $x$, e.g., $I(x)=H(x)$ when $x$ is a variable. Also, we have $I(v_1,v_2)=I(v_1)\cup\ I(v_2)$.
\end{definition}

\begin{definition}\label{Definition 1.2} 
The relative mutual information between two variables $v_1$ and $v_2$ is defined as the ratio of their mutual information to their total information:
\begin{equation}
    \widehat{M}(v_1,v_2)=\frac{M(v_1,v_2)}{I(v_1)\cup I(v_2)}=\frac{M(v_1,v_2)}{H(v_1)+H(v_2)-M(v_1,v_2)}
    \nonumber 
\end{equation}
Relative mutual information is more effective in highlighting the significance of shared information between two variables compared to conventional mutual information. 
\end{definition}

\begin{properties}\label{properties 1}
Properties of Mutual Information and Entropy:
\begin{flalign*}
\textbf{i}) \ &M(x;y)\geq 0, M(x;y|z)\geq 0. &
\end{flalign*}
\begin{flalign*}
\textbf{ii}) \ &M(x;y,z) = M(x;y)+M(x;z|y). &
\end{flalign*}
\begin{flalign*}
\textbf{iii}) \ M(x_1;x_2;\cdots;x_{n+1}) &= M(x_1;\cdots;x_n)\\ &-M(x1;\cdots;x_n|x_{n+1}). &
\end{flalign*}
\begin{flalign*}
\textbf{iv})\ \text{If} \ &I(v_2)\subseteq I(v_1)\longrightarrow M(v_1,v_2)=H(v_2), \\ &I(v_1,v_2)= I(v_1)\cup I(v_2)=I(v_1)=H(v_1) &
\end{flalign*}
\begin{flalign*}
\textbf{v}) \ \text{If}\ &I(v_2)\cap I(v_1)=\varnothing \longrightarrow \\ &I(v_1,v_2)=H(v_1,v_2)=H(v_1)+H(v_2)=I(v_1)+ I(v_2) &
\end{flalign*}
\end{properties}

\begin{proof}
    The proofs of properties \textbf{i, ii,} and \textbf{iii} can be found in \cite{cover1999elements}. For property \textbf{iv}:
    \begin{equation}
        \begin{split}
            M(v_1,v_2)&=\underset{v_1,v_2}{\iint }p(v_1,v_2)\mathrm{log}(\frac{p(v_1,v_2)}{p(v_1)p(v_2)})\\ &=\underset{v_1,v_2}\iint p(v_1,v_2)\mathrm{log} (\frac{\overbrace{p(v_2|v_1)}^{=1\ \text{as}\ I(v_2)\subseteq I(v_1)}p(v_1)}{p(v_1)p(v_2)})\\
&=\underset{v_2}{\int }-p(v_2)\mathrm{log}(p(v_2))=H(v_2).
        \end{split}
    \end{equation}\par
In addition, for $I(v_1,v_2)$, we have:
\begin{equation}
    \begin{split}
            I(v_1,v_2)&=I(v_1)\cup I(v_2)=H(v_1,v_2)\\&=\underset{v_1,v_2}{\iint }-p(v_1,v_2)\mathrm{log}( p(v_1,v_2))\\ 
            &=\underset{v_1,v_2}\iint -p(v_1,v_2)\mathrm{log}(p(v_2|v_1)p(v_1))\\
            &=\underset{v_1}\int -p(v_1)\mathrm{log}(p(v_1))=H(v_1)=(v_1).   
    \end{split}
\end{equation}\par
    For property \textbf{v}, we first clarified that:
    \begin{equation}
         I(v_2)\cap I(v_1)=\varnothing  \longrightarrow p(v_1,v_2)=p(v_1)p(v_2)
    \end{equation}\par
    Therefore, we have:
    \begin{equation}
        \begin{split}
            H(v_1,v_2)&=\underset{v_1,v_2}{\iint }-p(v_1,v_2)\mathrm{log} (p(v_1,v_2))\\ 
            &=\underset{v_1,v_2}{\iint }-p(v_1)p(v_2)\mathrm{log} (p(v_1)p(v_2))\\
            &=\underset{v_1}{\int }-p(v_1)\mathrm{log} (p(v_1))+\underset{v_2}{\int }-p(v_2)\mathrm{log}(p(v_2))\\
            &=H(v_1)+H(v_2)
        \end{split}
    \end{equation}
\end{proof}
\begin{assumption}\label{Assumption 1.1} 
Assume that histology image and ST represent two views ($v_1$ and $v_2$) of the biological information ($b$) inherent in the studied tissue. Let $y$ be an indicator of the normality of regions across the tissue, which is essentially determined by $b$. Then, we have:
 \begin{equation}
 \begin{split}
      &I(y)=\{b\}=\{b_0,b_1,b_2\},\\ 
      &\{b_0\}\cap\{b_1\}=\varnothing,\{b_0\}\cap\{b_2\}=\varnothing, \{b_1\}\cap\{b_2\}=\varnothing,\\ 
      &M(y;v_1)=\{b\}\cap I(v_1)= \{b_0,b_1\},\\
      &M(y;v_2)=\{b\}\cap I(v_2)= \{b_0,b_2\}.\\
 \end{split}
 \nonumber 
 \end{equation}
 
 Here, $b_0$ represents the common task-relevant information, while $b_1$ and $b_2$ represent the task-relevant information specific to $v_1$ and $v_2$, respectively. 
\end{assumption} 

\begin{assumption}\label{Assumption 1.2} 
The encodings $z_1=f_1(v_1)$, $z_2=f_2(v_2)$, and $ z_3=f_3(z_1,z_2)$ are generated by the respective encoders. We define $z_4=\{z_1,z_3\}$ and $z_5=\{z_2,z_3\}$ as per equation (6) in the main text. Assuming $f_1$ and $f_2$ are information lossless encoders, and, along with the fusion bottleneck encoder $f_3$, follow the information bottleneck theory proposed by Tishby et al., \cite{tishby2000information}. That is, $z_4\ \text{and}\ z_5$ should be maximally informative about $y$ with an information constrain on the bottleneck $z_3$. We use relative mutual information in place of conventional mutual information for more accurate reflection of the significance of shared information. The optimization problems are defined as:
\begin{equation}
\max_{f_1,f_3}\ \widehat{M}(z_4;y|f_1)\quad \mathrm{s.t.}\ \widehat{M}(z_3;v_1|f_3)\le I_c,
\nonumber 
\end{equation}
\begin{equation}
\max_{f_2,f_3}\ \widehat{M}(z_5;y|f_2)\quad \mathrm{s.t.}\ \widehat{M}(z_3;v_2|f_3)\le I_c,
\nonumber 
\end{equation}\par
where $I_c$ is the information constraint. These can be converted into the following objective functions by introducing a Lagrange multiplier $\beta>0$:
\begin{equation}
    \min_{z_3,z_4} \ell(z_3,z_4)=\min_{z_3,z_4}-\widehat{M}(z_4;y)+\beta \widehat{M}(z_3;v_1),
    \nonumber 
\end{equation}
\begin{equation}
    \min_{z_3,z_5} \ell(z_3,z_5)=\min_{z_3,z_5}-\widehat{M}(z_5;y)+\beta \widehat{M}(z_3;v_2).
    \nonumber 
\end{equation}
\end{assumption}

\begin{figure}[t]
  \centering
    \includegraphics[width=0.4\textwidth]{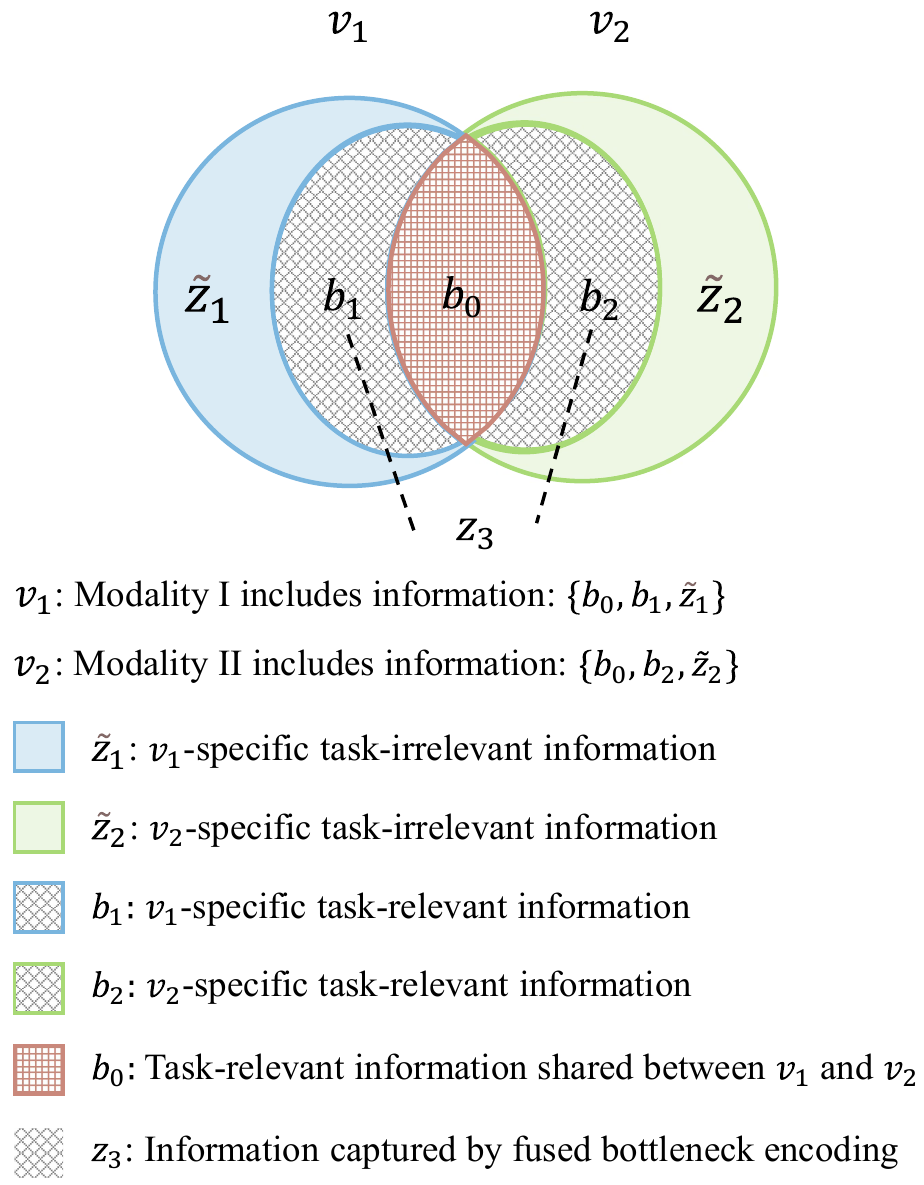}
  \caption{Information diagrams of the two data modalities $v_1$ and $v_2$. 
 The bottleneck encoding,  generated by the MGDAT block, embodies the minimally sufficient representation for modality-specific, task-relevant information (i.e., $b_1+b_2$).}
  \label{fig_com}
\end{figure}

\begin{proposition}\label{Proposition inclusiveness}  
\textbf{Inclusiveness of complementary task-relevant information.} The objective functions in \Cref{Assumption 1.2} are optimized when the bottleneck encoding $z_3$ encompasses all task-relevant information specific to $v_1$ and $v_2$:
\begin{equation}
I(z_3)\supseteq\{b_1,b_2\}
\nonumber 
\end{equation} 
\end{proposition}

\begin{proof}
Given that $f_1$ and $f_2$ are information lossless encoders of $v_1$ and $v_2$, we have:
\begin{equation}
    I(v_1)= I(z_1)= \{b_0,b_1,\tilde{z}_1\}
\end{equation}
\begin{equation}
    I(v_2)= I(z_2)= \{b_0,b_2,\tilde{z}_2\}
\end{equation}\par
 
 Here, $\tilde{z}_1$ and $\tilde{z}_2$ represents task-irrelevant information specific to $v_1$ and $v_2$, respectively. $b_0, \tilde{z}_1\ \text{and}\  \tilde{z}_2$ are mutually exclusive, i.e.,

\begin{equation}
\begin{aligned}
        &\{b_i\}\cap \{\tilde{z}_1\}=\varnothing, \{b_i\}\cap \{\tilde{z}_2\}=\varnothing,\\
        &\{\tilde{z}_1\}\cap \{\tilde{z}_2\}=\varnothing,\{b_i\}\cap \{b_j\}=\varnothing,\\  
        &\forall i,j\in \{0,1,2\}, i\neq j 
\end{aligned}
\end{equation}

Let $\{\check{z}_3\}=(\{b_0,b_1,b_2\}/(\{b_0,b_1,b_2\}\cap I(z_3)))\cap \{b_2\}$ represent the task-relevant information in $b_2$ that is not included in $z_3$. It is obvious:
\begin{equation}
    \begin{aligned}
        &\{\check{z}_3\}\subset I(y),\{\check{z}_3\} \cap I(z_3)=\varnothing,\\ 
        &\{\check{z}_3\}\cap \{b_0\}= \varnothing,\{\check{z}_3\}\cap \{b_1\}= \varnothing,\\ 
        &\{\check{z}_3\}\cap I(v_1)=\varnothing, \{\check{z}_3\}\cap I(z_1)=\varnothing. 
    \end{aligned}
\end{equation}
 
If $\{\check{z}_3 \}\neq \varnothing$, we have:
\begin{equation}
\begin{split}
   &\ell(z_3,z_4) = -\widehat{M}(z_4;y)+\beta\widehat{M}(z_3;v_1)\\ 
   &=-\widehat{M}(z_1,z_3;y)+\beta \widehat{M}(v_1;z_3)\\
   &=-\widehat{M}(z_1,z_3;y)+\beta (\frac{M(v_1;z_3)}{I(v_1)\cup I(z_3)}+\frac{\overbrace{M(v_1;\check{z}_3|z_3)}^{=0}}{I(v_1)\cup I(z_3)})\\
   &>-\widehat{M}(z_1,z_3;y)+\beta \frac{M(z_3,\check{z}_3;v_1)}{I(v_1)\cup \underbrace{I(z_3,\check{z}_3)}_{=H(z_3)+H(\check{z}_3)>H(z_3)=I(z_3)}}\\
   &=-\widehat{M}(z_1,z_3;y)+\beta \widehat{M}(z_3,\check{z}_3;v_1)\\
   \end{split}
   \nonumber 
\end{equation}\par

For $\widehat{M}(z_1,z_3;y)$, we have:
\begin{equation}
    \begin{split}
\widehat{M}(z_1,z_3;y) 
&=\frac{M(z_1,z_3;y)}{I(z_1,z_3) \cup I(y)}\\
&=\frac{M(z_1,z_3;y)}{I(z_1,z_3) \cup\underbrace{( I(y)\cup I(\check{z}_3))}_{=H(y)=I(y)\ \text{as}\ \{\check{z}_3\}\subset I(y)}}\\
&<\frac{\overbrace{M(y;\check{z}_3|z_1,z_3)}^{>0}+M(y;z_1,z_3)}{I(z_1,z_3,\check{z}_3)\cup I(y)}\\
&=\frac{M(y;z_1,z_3,\check{z}_3)}{I(z_1,z_3,\check{z}_3) \cup I(y)} =\widehat{M}(z_1,z_3, \check{z}_3;y)\\
\end{split}
\nonumber 
\end{equation}\par
Thus, $f_3$ will be updated to generate $z'_{3}\ \text{with}\ I(z'_3)=\{I(z_3),\check{z}_3\}$ so that:
\begin{equation}
\begin{split}
   \ell(z'_3,z_4) 
   &=-\widehat{M}(z_1,z'_3;y)+\beta \widehat{M}(z'_3;v_1)\\
   &=-\widehat{M}(z_1,z_3, \check{z}_3;y)+\beta \widehat{M}(z_3,\check{z}_3;v_1)\\
   &<\ell(z_3,z_4) 
\end{split}
\end{equation}
This update continues until $\{\check{z}_3\}=\varnothing\rightarrow\{b_2\}\subseteq I(z_3)$. Similarly, using $\ell(z_3,z_5)$, we can show that $\{\hat{z}_3\}= (\{b_0,b_1,b_2\}/(\{b_0,b_1,b_2\}\cap I(z_3))\cap \{b_1\}=\varnothing \rightarrow \{b_1\}\subseteq I(z_3)$. Therefore, $I(z_3)\supseteq\{b_1,b_2\}$, completing the proof.
\end{proof}

\begin{proposition}\label{Proposition compactness}  
\textbf{Compactness of complementary task-relevant information.} The objective functions in \Cref{Assumption 1.2} is minimized when:
\begin{equation}
I(z_3)=\{b_1,b_2\}
\nonumber 
\end{equation} 
\end{proposition}

\begin{proof} 
As proved in \Cref{Proposition inclusiveness}, $I(z_3)\supseteq\{b_1,b_2\}$. We start with $I(z_3)=\{b_0,b_1,b_2\}$, and then expand $z_3$ to encompass additional information from $v_1$, denoted as $\{\check{z}_3\}$, where: 

\begin{equation}
    \begin{aligned}
        &\{\check{z}_3\}\subset I(v_1)=I(z_1), M(\check{z}_3,v_1)>0,\\ 
        &M(\check{z}_3,y)=0, \{z_3\}\cap \{\check{z}_3\}=\varnothing.
    \end{aligned}
\end{equation}

Let $\ddot{z}_3=\{z_3,\check{z}_3\}$. The objective function becomes:

\begin{equation}
\begin{split}
   \ell(\ddot{z}_3,z_4) &= -\widehat{M}(z_4;y)+\beta\widehat{M}(\ddot{z}_3;v_1)\\ 
   &=-\widehat{M}(z_1,z_3,\check{z}_3;y)+\beta \widehat{M}(\check{z}_3,z_3;v_1)\\
   \end{split}
   \end{equation}\par
For $\widehat{M}(\check{z}_3,z_3;v_1)$, we have:
\begin{equation}
    \begin{split}
        \widehat{M}(\check{z}_3,z_3;v_1)&=\frac{M(v_1;z_3)+\overbrace{M(v_1;\check{z}_3|z_3)}^{>0}}{\underbrace{I(v_1)\cup I(z_3)}_{\because I(z_3,\check{z}_3)=I(z_3)\cup I(\check{z}_3);\  I(\check{z}_3)\cup I(v_1)=I(v_1)} }\\ 
        &>\frac{M(z_3;v_1)}{I(v_1)\cup I(z_3) }=\widehat{M}(z_3;v_1)
    \end{split}
\end{equation} \par
For $\widehat{M}(z_1,z_3,\check{z}_3;y)$, we have:
\begin{equation}
    \begin{split}
       \widehat{M}(z_1,z_3,\check{z}_3;y)&=\frac{\overbrace{M(y;\check{z}_3|z_1,z_3)}^{=0}+M(y;z_1,z_3)}{I(z_1,z_3,\check{z}_3)\cup I(y) }\\
        &=\frac{M(y;z_1,z_3)}{\underbrace{I(z_1,z_3)}_{\because \{\check{z}_3\}\subset I(z_1)}\cup I(y) }=\widehat{M}(z_1,z_3;y)
    \end{split}
\end{equation}
Thus, $\ell(\ddot{z}_3,z_4)>-\widehat{M}(z_1,z_3;y)+\beta \widehat{M}(z_3;v_1)=\ell(z_3,z_4), \forall \beta>0$. To minimize the objective function, $\check{z}_3 $ is excluded. Similarly, from $\ell(z_3,z_5)$, we know $z_3$ should not expand to encompass additional information from $v_2$. Hence, optimal $z_3$ must satisfy $I(z_3)\subseteq\{b_0,b_1,b_2\}$.

Furthermore, if we shrink the information of $z_3$ to $\dot{z}_3$, with the reduced information $\{\hat{z}_3\}\subseteq\{b_0\}\subset I(z_1)=I(v_1)\rightarrow \{\dot{z}_3\}\cap\{\hat{z}_3\}=\varnothing$. The objective function becomes:
\begin{equation}
\begin{split}
   &\ell(z_3,z_4) = -\widehat{M}(z_4;y)+\beta\widehat{M}(z_3;v_1)\\ 
   &=-\widehat{M}(z_1,\dot{z}_3,\hat{z}_3;y)+\beta \widehat{M}(\dot{z}_3,\hat{z}_3;v_1)\\
   &=-\frac{M(y;z_1,\dot{z}_3)+\overbrace{M(y;\hat{z}_3|z_1,\dot{z}_3)}^{=0\ as\ \hat{z}_3\subset I(z_1)}}{I(z_1,\dot{z}_3,\hat{z}_3)\cup I(y)}\\ 
   &+\beta \frac{M(\dot{z}_3,\hat{z}_3;v_1)}{(I(\dot{z}_3)\cup I(\hat{z}_3))\cup I(v_1)} \\
   &=-\frac{M(z_1,\dot{z}_3;y)}{\underbrace{I(z_1,\dot{z}_3)}_{\because \{\hat{z}_3\}\subset I(z_1)}\cup I(y)}+\beta \frac{M(v_1;\dot{z}_3)+\overbrace{M(v_1;\hat{z}_3|\dot{z}_3)}^{>0}}{I(\dot{z}_3)\cup I(v_1)}\\
   &>-\frac{M(z_1,\dot{z}_3;y)}{I(z_1,\dot{z}_3)\cup I(y)}+\beta \frac{M(\dot{z}_3;v_1)}{I(\dot{z}_3)\cup I(v_1)}=\ell(\dot{z}_3,z_4)
   \end{split}
   \nonumber 
   \end{equation}
Therefore, if $M(v_1;\hat{z}_3)> 0$, the objective function can be further optimized by reducing information from $\{b_0\}$ until $I(z_3)\cap\{b_0\}=\varnothing\rightarrow I(z_3)=\{b_1,b_2\}$. This completes the proof.

\end{proof}
In summary, $f_3$ effectively captures the view-specific, task-relevant information in the bottleneck encoding $z_3$, which embodies an inclusive and condensed representation of the complementary information, biologically relevant for determining tissue region normality, between the two data modalities. Thus, $z_3$ serves as an informational bridge connecting the two data modalities.

\begin{table*}[t]
\centering
    \fontsize{7.8pt}{10.5pt}\selectfont
    \renewcommand{\arraystretch}{1.25}
    \setlength{\tabcolsep}{3.55pt} 

\label{table:data}
\resizebox{\textwidth}{!}{
\begin{tabular}{l|l|l|l}
\Xhline{1pt}
Dataset & Tissue (ATR Type)& Total Number of Spots& Anomaly Proportion\\ \Xhline{0.7pt}
10x-hNB-v03 & Normal human breast & 2364 & 0.00\%  \\ \Xhline{0.5pt}
10x-hNB-v04 & Normal human breast & 2504 & 0.00\%  \\ \Xhline{0.5pt}
10x-hNB-v05 & Normal human breast & 2224 & 0.00\%  \\ \Xhline{0.5pt}
10x-hNB-v06 & Normal human breast & 3037 & 0.00\%  \\ \Xhline{0.5pt}
10x-hNB-v07 & Normal human breast & 2086 & 0.00\%  \\ \Xhline{0.5pt}
10x-hNB-v08 & Normal human breast & 2801 & 0.00\%  \\ \Xhline{0.5pt}
10x-hNB-v09 & Normal human breast & 2694 & 0.00\%  \\ \Xhline{0.5pt}
10x-hNB-v10 & Normal human breast & 2473& 0.00\%  \\ \Xhline{0.5pt}
10x-hBC-A1 & Human breast cancer (Cancer in situ, Invasive cancer) & 346 & 12.43\% \\ \Xhline{0.5pt}
10x-hBC-B1 & Human breast cancer (Invasive cancer) & 295 & 78.64\% \\ \Xhline{0.5pt}
10x-hBC-C1 & Human breast cancer (Invasive cancer) & 176 & 27.84\%  \\ \Xhline{0.5pt}
10x-hBC-D1 & Human breast cancer (Invasive cancer) & 306 & 54.58\%  \\ \Xhline{0.5pt}
10x-hBC-E1 & Human breast cancer (Invasive cancer) & 587 & 42.08\% \\ \Xhline{0.5pt}
10x-hBC-F1 & Human breast cancer (Invasive cancer) & 691 & 16.50\%  \\ \Xhline{0.5pt}
10x-hBC-G2 & Human breast cancer (Cancer in situ, Invasive cancer) & 467 & 65.74\%  \\ \Xhline{0.5pt}
10x-hBC-H1 & Human breast cancer (Cancer in situ, Invasive cancer) & 613 & 69.49\%  \\ \Xhline{0.5pt}
10x-hBC-I1 & Human breast cancer (Ductal carcinoma in situ, Lobular carcinoma in situ, Invasive Carcinoma) & 1308 & 62.92\% \\ \Xhline{0.5pt}
10x-hLiver-A1 &Healthy human liver &2378 &0.00\% \\ \Xhline{0.5pt}
10x-hLiver-B1 &Healthy human liver &2349 &0.00\% \\ \Xhline{0.5pt}
10x-hLiver-C1 &Healthy human liver &2277 &0.00\% \\ \Xhline{0.5pt}
10x-hLiver-D1 &Healthy human liver &2265 &0.00\% \\ \Xhline{0.5pt}
10x-PSC-A1 &PSC human liver (native cell, intrahepatic cholangiocyte) &3118 &26.36\% \\ \Xhline{0.5pt}
10x-PSC-B1 &PSC human liver (PSC fibrotic region) &2670 &24.91\% \\ \Xhline{0.5pt}
10x-PSC-C1 &PSC human liver (native cell, intrahepatic cholangiocyte) &3322 &25.89\% \\ \Xhline{0.5pt}
10x-PSC-D1 &PSC human liver (PSC fibrotic region) &3174 &25.65\% \\
\Xhline{1pt}
\end{tabular}}
\caption{Overview of the experimental datasets.}
\label{table:dataset}
\end{table*}

\begin{table*}[t]
\centering
\renewcommand{\arraystretch}{1.25}
\fontsize{10pt}{14pt}\selectfont
\setlength{\tabcolsep}{5.1pt}
\resizebox{\textwidth}{!}{
\begin{tabular}{c"c"cc"ccc"ccccc}
\Xhline{1pt}
\multirow{3}{*}{\centering \makecell{Target \\ Dataset}}& \multirow{3}{*}{\centering Metric}& \multicolumn{9}{c}{\large Method}\\

\Xcline{3-12}{0.9pt}
& & \multicolumn{2}{c"}{Multimodal-based} & \multicolumn{3}{c"}{Image-based} & \multicolumn{5}{c}{ST-based} \\ 
\Xcline{3-12}{0.9pt}
& & MEATRD & M3DM  & SimpleNet & f-AnoGAN & PatchSVDD & DOMINANT & PREM & Spatial-ID & scmap & CAMLU \\ 
\Xhline{0.7pt}
\multirow{2}{*}{10x-PSC-A1} & AUC & $\mathbf{0.657}_{\pm0.073}$ & $0.475_{\pm0.006}$ & $0.483_{\pm0.129}$ & $\underline{0.647}_{\pm0.001}$ & - &  $0.590_{\pm0.043}$ & $0.567_{\pm0.008}$ & $0.486_{\pm0.006}$ & $0.500_{\pm0.000}$ & $0.537_{\pm0.074}$ \\

& F1 & $\mathbf{0.629}_{\pm0.085}$ & $0.235_{\pm0.007}$ & $0.284_{\pm0.114}$ & $\underline{0.479}_{\pm0.001}$ & - &  $0.356_{\pm0.046}$ & $0.291_{\pm0.010}$ & $0.449_{\pm0.009}$ & $0.415_{\pm0.000}$ & $0.217_{\pm0.118}$ \\
\Xhline{0.7pt}
\multirow{2}{*}{10x-PSC-B1} & AUC &  $\mathbf{0.675}_{\pm0.092}$ & $0.475_{\pm0.006}$ & $0.481_{\pm0.161}$ & $\underline{0.655}_{\pm0.001}$ & -  & $0.547_{\pm0.040}$ & $0.520_{\pm0.030}$ & $0.519_{\pm0.030}$ & $0.500_{\pm0.000}$ & $0.503_{\pm0.004}$ \\ 

& F1 &  $\mathbf{0.646}_{\pm0.101}$ & $0.235_{\pm0.007}$ & $0.274_{\pm0.146}$ & $0.482_{\pm0.001}$ & -  & $0.307_{\pm0.031}$ & $0.231_{\pm0.005}$ & $0.464_{\pm0.037}$ & $\underline{0.625}_{\pm0.000}$ & $0.017_{\pm0.018}$ \\ 
\Xhline{0.7pt}	
\multirow{2}{*}{10x-PSC-C1} & AUC & $\underline{0.664}_{\pm0.069}$ & $0.497_{\pm0.013}$ & $0.468_{\pm0.156}$ & $\mathbf{0.683}_{\pm0.001}$ & - & $0.595_{\pm0.060}$ & $0.508_{\pm0.008}$ & $0.517_{\pm0.015}$ & $0.500_{\pm0.000}$ & $0.501_{\pm0.002}$ \\ 

& F1 & $\mathbf{0.631}_{\pm0.078}$ & $0.259_{\pm0.012}$ & $0.268_{\pm0.131}$ & $\underline{0.530}_{\pm0.001}$ & -  & $0.344_{\pm0.057}$ & $0.245_{\pm0.006}$ & $0.470_{\pm0.012}$ & $0.411_{\pm0.000}$ & $0.008_{\pm0.006}$ \\
\Xhline{0.7pt}
\multirow{2}{*}{10x-PSC-D1} & AUC & $\mathbf{0.655}_{\pm0.071}$ & 
$0.498_{\pm0.012}$ &
$0.473_{\pm0.182}$ & $\underline{0.639}_{\pm0.001}$ & - & $0.534_{\pm0.091}$ & $0.519_{\pm0.001}$ & $0.575_{\pm0.018}$ & $0.500_{\pm0.000}$ & $0.508_{\pm0.011}$ \\

& F1 & $\mathbf{0.627}_{\pm0.082}$ & $0.266_{\pm0.013}$ & $0.287_{\pm0.172}$ & $0.464_{\pm0.001}$ & -  & $0.290_{\pm0.071}$ & $0.229_{\pm0.001}$ & $\underline{0.512}_{\pm0.014}$ & $0.408_{\pm0.000}$ & $0.042_{\pm0.036}$ \\
\Xhline{0.7pt}
\multirow{2}{*}{Mean} & AUC & $\mathbf{0.663}$ & 
$0.486$ &
$0.476$ & $\underline{0.656}$ & - & $0.567$ & $0.529$ & $0.524$ & $0.500$ & $0.512$ \\

& F1 & $\mathbf{0.633}$ & $0.249$ & $0.278$ & $\underline{0.489}$ & -  & $0.324$ & $0.249$ & $0.474$ & $0.465$ & $0.071$ \\
\Xhline{1pt}
\end{tabular}
}
\caption{Performance evaluation of anomalous tissue region detection across four primary sclerosing cholangitis (PSC) liver ST datasets. The table presents the results in terms of AUC and F1 scores, with each cell showing the average score from five independent runs and the corresponding standard deviation. The best score for each dataset is \textbf{bolded}, and the second-best score is \underline{underline}.}
\label{table:Main2}
\end{table*}

\section{Implementation}

\subsection{Dataset Descriptions}\label{dataset_descriptions}

As detailed in \Cref{table:dataset}, we conducted extensive experiments on two types of disease datasets to validate the generalizability of MEATRD: \\
\textbf{Breast Cancer}: ST datasets about Breast Cancer used in this study include eight reference 10x Visium datasets \cite{kumar2023spatially} derived from human healthy breast tissues, denoted as 10x-hNB-\{v03-v10\}\footnote{\url{https://cellxgene.cziscience.com/collections/4195ab4c-20bd-4cd3-8b3d-65601277e731}}, and nine target 10x Visium datasets \cite{andersson2021spatial} derived from human breast cancer tissues, denoted as 10x-hBC-\{A1-I1\}\footnote{\url{https://github.com/almaan/her2st}, and \url{https://zenodo.org/records/10437391}}.
\\ 
\textbf{Primary sclerosing cholangitis (PSC)}\footnote{\url{https://cellxgene.cziscience.com/collections/0c8a364b-97b5-4cc8-a593-23c38c6f0ac5}}
: Reference 10x Visium datasets denoted as 10x-hLiver-\{A1-D1\} contains 4 healthy human liver datasets and target 10x Visium datasets denoted as 10x-PSC-\{A1-D1\} are collected from 4 Primary sclerosing cholangitis slices. \\
The reference datasets are collectively used during training, and the target datasets are used during inference only. For each ST dataset, genes detected in fewer than 10 spots are excluded \cite{wolf2018scanpy}.
\subsection{Data Preprocessing}\label{preprocess}
For each ST dataset, genes detected in fewer than 10 spots are excluded. Then, raw gene expression counts are normalized with library size and log-transformed. 3000 highly variable genes (HVGs) are selected as inputs to the model using using the SCANPY package ~\cite{wolf2018scanpy}. 
 
\subsection{Implementation Details}\label{baseline details}
MEATRD is implemented using PyTorch \cite{paszke2019pytorch}. In Stage I, we adopt the default architecture of the Mobile-UNet with an output embedding dimension of 256. In Stage II, the gene encoder is a two-layer MLP, while the image encoder is the frozen Mobile-Unet encoder from Stage I. The MGDAT network includes three MGDAT blocks, each having a two-layer, four-headed transformer to generate 16-dimensional fused bottleneck embeddings in equation (5), and an attention layer with an input dimension of 272 and an output dimension of 256 in equation (8). The ResNet-based image decoder in this stage comprises eight residual blocks, while the gene decoder is a single-layer GNN with an output dimension of 3000. In Stage III, the image encoder is an eight-layer ResNet, while the gene encoder is consistent with that in Stage II. The FFN for multimodal data fusion is structured as a two-layer MLP with an output dimension of 256. In all stages, the training is conducted with the Adam optimizer, with a batch size to 128, and a learning rate of 1e-4. Stage I is trained for 30 epochs, Stage II for 10 epochs, and Stage III for 5 epochs. Finally, we have the three weight parameters $\alpha=0.5$ and $\beta=1$. 
For all baselines but M3DM and Spatial-ID, we adopted the recommended or default settings in the original study. M3DM, designed for natural images and point clouds, has an encoder unsuitable for ST data and histology image. Therefore, we replaced its encoders with MEATRD's encoders for fair comparison. Since Spatial-ID is pretrained using single-cell sequencing (scRNA-seq) data, we skipped its pretraining step and directly utilizes the pretrained model.

\section{Further Experiments}
\subsection{Supplement Results of Anomalous Tissue Region Detection}
In this section, we present the performance of MEATRD in testing for primary sclerosing cholangitis. MEATRD is trained on four human healthy liver ST datasets (i.e., 10x-hLiver-\{A1-D1\}) and tested on four human PSC (i.e., 10x-PSC-\{A1-D1\}) ST datasets. \Cref{table:Main2} highlights MEATRD's superiority over baseline models in detecting ATRs across datasets, consistently achieving the highest AUC scores and three times ranking first in F1 scores. The experiment yields similar results in terms of AP scores, as shown in Table 3 in supplementary material F.

\subsection{Ablation Studies}

\begin{table}[t]
\centering
\renewcommand{\arraystretch}{1.2}
\fontsize{7pt}{8pt}\selectfont

\begin{tabular}{c|c|c|c|c|c}
\Xhline{0.8pt}
Backbones   & Params & AUC & F1 & \makecell[c]{Training\\ time (min)} & \makecell[c]{Inference\\ time (s)} \\ \Xhline{0.5pt}
\textbf{MobileUNet}                 &\textbf{46.17M}		&\textbf{0.723}     &\textbf{0.741}    &\textbf{11.984}               &\textbf{0.354}                \\
             ResNet-18                     &53.71M           &0.717     &0.729    &12.927               &0.455                \\
             VGG-19                        &133.402M            &0.696     &0.703    &19.887              &0.489                \\ \Xhline{0.5pt}
MoCo                       &53.71M            &0.712     &0.726     &26.435               &0.449                \\
\Xhline{0.8pt}
\end{tabular}
\caption{Ablation study of backbones in MEATRD across eight human breast cancer datasets.}
\label{table:backbone}
\end{table}

\begin{table*}[t]
\centering

\renewcommand{\arraystretch}{1.25}
\fontsize{7pt}{8pt}\selectfont
\begin{tabular}{c|c|c|c|c|c|c|c}
\Xhline{1pt} 
\multicolumn{2}{c|}{Model}  & Params& MFlops & Complexity& Training time (s) & Inference time (s) & Memory Usage (GB) \\ \Xhline{0.7pt}
\multicolumn{1}{c|}{\multirow{3}{*}{MEATRD}}    &Stage I&0.73M&32.746&$\mathcal{O}(c_1|V|)$&100.02                & -   & 1.70    \\
    &Stage II&46.17M&405.909&$\mathcal{O}(c_2|V|+c_3|E|+c_4|D|^2)$&469.02                &0.15    & 4.65   \\
    &Stage III&5.22M&29.993&$\mathcal{O}(c_5|V|)$&150.00                &0.21   & 5.15    \\ \Xhline{0.7pt}
\multicolumn{2}{c|}{M3DM}       &            97.37M&   8,028.937& -              &468.00                &20.75   & 1.60    \\
\multicolumn{2}{c|}{SimpleNet}  &72.82M            &238.954& -               &0.72                &7.72   & 0.57\\
\multicolumn{2}{c|}{f-AnoGAN}   &1.30M            &118.496& -               &1322.80                &3.82      & 0.14 \\
\multicolumn{2}{c|}{PatchSVDD}  &0.17M&1.083& -               &5761.08                &584.33     & 0.16  \\
\multicolumn{2}{c|}{PREM}       &0.38M&0.768& -               &326.58                &0.01    & 0.02   \\
\multicolumn{2}{c|}{DOMINANT}   &0.40M&0.396& -               &0.83                &0.01     & 11.34  \\
\multicolumn{2}{c|}{Spatial-ID} &4.36M&20.555& -               &501.07&0.64 & 0.21\\
\multicolumn{2}{c|}{Scmap}      & -           & -  & -               &2.02                &0.16      & - \\
\multicolumn{2}{c|}{CAMLU}      &1.62M&1.619& -               &125.57                &0.64      & - \\ \Xhline{1pt}
\end{tabular}
\caption{The overall training time on eight 10x-hNB datasets, including a total of 20,183 spots. Each spot is associated with a 3000-dimensional gene expression vector and a histology image patch of size 32x32.}
\label{complexity}
\end{table*}

\textbf{Using multiple data modality.}
In this evaluation, MEATRD is adapted to use either histology image or ST data alone, by skipping the step of the multimodal data fusion process and omitting the branch for learning the alternative data modality. Using only histology images results in a substantial reduction of MEATRD's average AUC scores by 31.26\% and F1 scores by 26.59\%, while using ST data alone on average lowers its AUC scores by 12.72\% and F1 scores by 9.99\%. These findings corroborate the synergic effects of histology image and ST data in enhancing ATR detection.\\
\textbf{Multimodal data fusion using fused bottleneck embedding.}
To assess the impact of multimodal data fusion on ATR detection, we substitute the cross-modal bottleneck embedding-guided fusion with a direct concatenation of image and gene embeddings. The comparison reveals that using bottleneck embedding for multimodal data fusion contributes to an average increase of 13.15\% in AUC scores and 7.55\% in F1 scores over the simple concatenation method. This improvement underscores our approach's efficacy in enhancing multimodal embeddings by collating and condensing the most relevant information from each data modality. \\
\textbf{Masking for target node reconstruction.}
Introducing target-node-masking, which strategically omits self-information during the target node reconstruction, theoretically mitigates the model's over-generalization issue. This masking prevents the model from "learning too well" to replicate its input, leading to minimal reconstruction errors even for anomalies. To validate the effectiveness of this technique, we compare the latent multimodal reconstruction errors, defined in equation (13), with and without target-node-masking during inter-node message passing. The violin plots in \Cref{suppfig1} showcase that this masking not only increases reconstruction errors for both inliers and anomalies but also amplifies the discrepancy between their reconstruction errors. This enhanced discrepancy in turn aids MEATRD's discriminative model in Stage III to separate anomalies from inliers, contributing to an average increase of 10.38\% in AUC scores and 6.01\% in F1 scores when compared to the omission of target-node-masking.  
\begin{figure}[h]
  \centering
\includegraphics[width=0.3\textwidth]{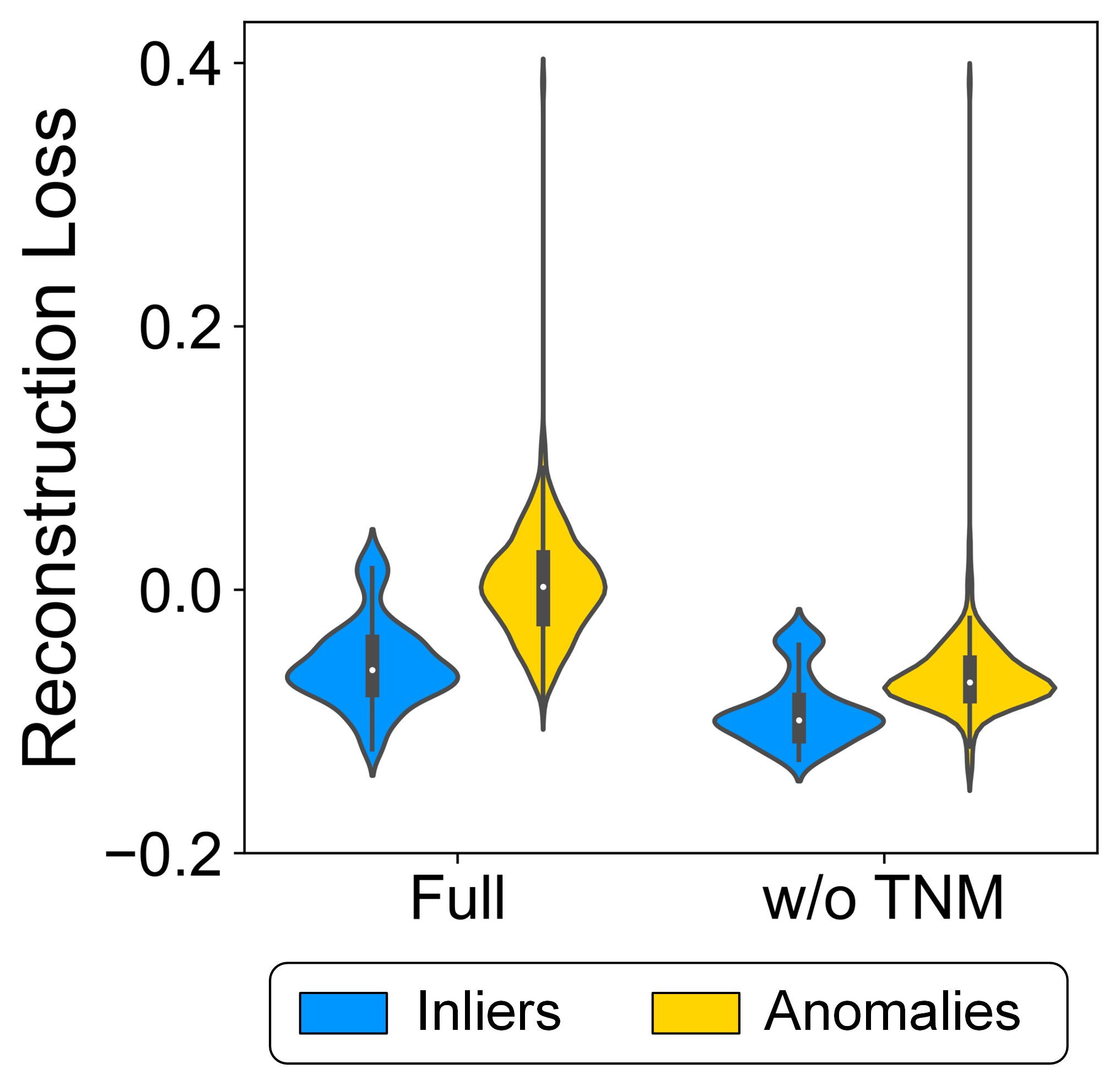}
  \caption{Violin plots illustrating the distributions of latent multimodal reconstruction errors for inliers (blue) and anomalies (yellow) with (``Full'') and without (``w/o TNM'') the implementation of target node masking (TNM).}
  \label{suppfig1}
\end{figure}
\\
\textbf{Multimodal reconstruction losses in one-class classifier.}
Most one-class classification methods directly utilize inlier embeddings in the reference to determine normal data distribution and thus rely heavily on the quality of instance embeddings, which, however, is sensitive to batch effects across datasets \cite{ouardini2019towards}. Here, we investigate whether using latent multimodal reconstruction losses, which avoid cross-batch comparisons, as an alternative input to METARD's one-class classifier can improve the accuracy of ATR detection. For this purpose, instance embeddings generated by the MGDAT network instead of the latent reconstruction losses are input to the one-class classifier in Stage III. We find that this modification reduces MEATRD's performance, with an average decline of 11.20\% in AUC scores and 7.56\% in F1 scores, demonstrating the necessity of using latent reconstruction losses in this context.
\\
\textbf{One-class classifier.}
Finally, to assess the effect of the one-class classification in ATR detection, we remove the entire stage III and use the weighted sum of image and gene reconstruction errors, defined in equation (10) in the main text, as anomaly scores. The direct use of reconstruction errors leads to MEATRD's suboptimal performance, as indicated by its average decrease of 19.23\% in AUC scores and 14.84\% in F1 scores. This finding suggests that, by collapsing multimodal reconstruction losses of inliers into a compact hypersphere in the latent space, the separation of inliers and anomalies is boosted, addressing the model over-generalization.
\\
\textbf{Mobine-Unet as pretrained visual feature extractor.}
Here, we replace Mobile-Unet with three pretrained visual feature extractor widely used for natural images, including VGG-19 ~\cite{simonyan2014very} and ResNet-18 ~\cite{he2016deep}, and MoCo ~\cite{he2020momentum}, to extract visual features from histology images of the eight human breast cancer datasets in Stage I. 
As shown in \Cref{table:backbone}, MEATRD's performance declines with these networks, as indicated by the lower AUC and F1 scores. This is probably due to that tissue images contain complex patterns and features specific to biological structures, which may not be effectively captured by networks optimized for natural image recognition. Particularly, data augmentation techniques, e.g., blurring and resizing, used by contrastive learning approach can generate "positive" images with semantics that significantly deviate from the original image.

\subsection{Sensitivity Analysis}

Table 2 shows the average model performance over five independent runs across the eight 10x-hBC datasets. We observe that a heavily weighing on image data (e.g., $\alpha=0.9$ or $\beta=0.9$ ) compromises model performance due to inadequate utilization of gene information for visually indistinguishable ATR in histology image. On the other hand, overly weighting ST data ($\alpha=0.1$ or $\beta=0.1$) also reduces performance, though it outperforms overweighing image data, likely due to the higher signal-to-noise ratio in ST data. Additionally, optimal model performance is achieved with a small bottleneck embedding dimension (e.g., 16), aligning with our theoretical analysis that nuisance information is minimized by in condensed bottleneck embedding. In contrast, a larger dimension for the one-class classification classifier (e.g., 256) is beneficial, providing more flexibility for collapsing inlier embeddings into a hypersphere. The best results are obtained with three MGDAT layers, balancing message passing within the graph and data over-smoothing. Lastly, variations in the number of attention heads in MGDAT and the visual feature dimension have relatively minor impact on model performance. 

\subsection{Complexity Analysis}
In this section, we first theoretically analyze MEATRD's model complexity.  As shown in \Cref{complexity}, Stage I is built on a 36-layer CNN network comprising lightweight inverted residual blocks, with 0.73M parameters and a time complexity of $\mathcal{O}(c_1|V|)$~\cite{he2015convolutional}, where $|V|$ denotes the number of instances. In Stage II, the main complexity arises from the MGDAT blocks, which compute feature-level and node-level attentions with complexities of $\mathcal{O}(c_2|V|)$ and $\mathcal{O}(c_3|E|)$, respectively ~\cite{velivckovic2018graph}. Here, $|E|$ denotes the number of graph edges and $|E|\ll |V|^2$. This stage has 46.17M parameters. Stage III mainly involves the lightweight ResNet, which has a time complexity of $\mathcal{O}(c_5|V|)$ ~\cite{he2015convolutional}, and it has 5.22M parameters.  

Empirical training on 20,183 image patches shows that Stage I has 32.746 MFlops and a training time of 100.02s, Stage II has 405.909 MFlops and a training time of 469.02s, stage III has 29.993 MFlops and a training time of 150s. Inference time is negligible compared to training time. MEATRD's total time cost is comparable to well-received methods M3DM and Spatial-ID.

\subsection{Robustness to Noisy Data}
The quality of ST data is indeed crucial for the model's performance, and $\beta$ can be used to balance the influence of different data sources accordingly. Specifically, when the quality of ST data is lower relative to image data, we set a lower $\beta$ to increase the model's reliance on image data, and vice versa. In our experiments with human breast tissue datasets (10x-hNB), where image and ST data have comparable quality, we set $\beta$ to 0.5. Typically, ST data quality is assessed using the average location-wise zero proportion $z$ (Zhu et al., 2023), representing the average proportion of zero-read-count genes, with lower values indicating higher signal-to-noise ratios. To explore a more systematic setting of, we altered $z$ of the 10x-hNB by randomly masking gene read counts and tested various $\beta$ values, as shown below:

\begin{table}[h]
\centering
\resizebox{\linewidth}{!}{
\begin{tabular}{ccccccc}
\toprule
$\bar{z}$ & 0.1 & 0.3 & 0.5 & 0.7 & 0.9 \\
\midrule
0.925 & 0.699 & 0.713 & \textbf{0.723} & 0.707 & 0.654 \\
0.950 & 0.655 & 0.696 & \textbf{0.712} & 0.701 & 0.657 \\
0.975 & 0.644 & 0.652 & 0.658 & \textbf{0.697} & 0.655 \\
\bottomrule
\end{tabular}
}
\end{table}

We find that, when $\bar{z}\leq 0.95$, default setting $\beta=0.5$ consistently yields the best results; only in the extreme case $\bar{z}= 0.975$ indicating the quality of gene data is too poor, it's necessary to set $\beta=0.7$. Considering this relationship, we also provide an adaptive strategy. We set the heuristic function for $\beta$ varying with $\bar{z}$ as:
\begin{equation*}
    \beta = \begin{cases}
        0.5, & \bar{z} \leq 0.95 \\
        0.5 + 0.5\textrm{sigmoid}\left(200(\bar{z}-0.975)\right), & \bar{z} > 0.95
    \end{cases}
\end{equation*}
 
\end{document}